\documentclass[12pt,reqno]{amsart}

\usepackage{amsmath,amssymb,amsthm}
\usepackage{amsaddr}
\usepackage{natbib}
\usepackage{mathtools}
\usepackage{mathrsfs}

\usepackage{fullpage}
\usepackage[utf8]{inputenc}
\usepackage[T1]{fontenc}
\usepackage{hyperref}
\usepackage{url}
\usepackage{booktabs}
\usepackage{amsfonts}
\usepackage{nicefrac}
\usepackage{microtype}
\usepackage{xcolor}

\usepackage{amsmath}
\usepackage{amssymb}
\usepackage{mathtools}
\usepackage{amsthm}
\usepackage{wrapfig}
\usepackage{caption}
\usepackage{float}

\theoremstyle{plain}
\newtheorem{theorem}{Theorem}[section]
\newtheorem{proposition}[theorem]{Proposition}
\newtheorem{lemma}[theorem]{Lemma}

\theoremstyle{definition}
\newtheorem{definition}[theorem]{Definition}
\newtheorem{assumption}[theorem]{Assumption}
\theoremstyle{remark}
\newtheorem{remark}[theorem]{Remark}
\newtheorem{example}{Example}

\usepackage[textsize=tiny]{todonotes}

\usepackage{amsmath,amssymb,amsthm}
\usepackage{bm}
\usepackage{natbib}
\usepackage{graphicx,here,comment}
\usepackage{algorithmic}
\usepackage{algorithm}
\usepackage{multirow}
\usepackage{xcolor}
\usepackage{wrapfig}
\usepackage{footnote}
\makesavenoteenv{table}

\usepackage{mathtools}

\usepackage{hyperref}

\newcommand{\rbr}[1]{\left(#1\right)}
\newcommand{\sbr}[1]{\left[#1\right]}
\newcommand{\cbr}[1]{\left\{#1\right\}}

\newcommand{\R}{\mathbb{R}}
\newcommand{\N}{\mathbb{N}}

\newcommand{\mA}{\mathcal{A}}

\newcommand{\mH}{\mathcal{H}}

\newcommand{\mN}{\mathcal{N}}
\newcommand{\mX}{\mathcal{X}}

\newcommand{\Ep}{\mathbb{E}}
\renewcommand{\Pr}{\mathbb{P}}

\renewcommand{\hat}{\widehat}
\renewcommand{\tilde}{\widetilde}

\newcommand{\argmin}{\operatornamewithlimits{argmin}}
\newcommand{\argmax}{\operatornamewithlimits{argmax}}

\newcommand{\mone}{\textbf{1}}

\def\bX{\mathbf{X}}
\def\bY{\mathbf{Y}}

\def\bI{\mathbf{I}}

\newcommand{\trace}{\mathrm{tr}}

\DeclareMathOperator\E{\mathbb{E}}
\DeclareMathOperator{\Var}{Var}

\usepackage{color}

\usepackage{newtxtext,newtxmath}

\newcommand{\secref}[1]{Section~\ref{#1}}

\newcommand{\algoref}[1]{Algorithm~\ref{#1}}
\newcommand{\asmpref}[1]{Assumption~\ref{#1}}
\newcommand{\thmref}[1]{Theorem~\ref{#1}}
\newcommand{\lemref}[1]{Lemma~\ref{#1}}

\newcommand{\exref}[1]{Example~\ref{#1}}
\newcommand{\Ich}{I(t)}

\newcommand{\Xit}{X^{(i)}(t)}

\newcommand{\Xarmt}[1]{X^{(#1)}(t)}

\newcommand{\Xt}{X^{(I(t))}(t)}
\newcommand{\Yt}{Y^{(I(t))}(t)}

\newcommand{\Xin}[1]{X^{(i)}_{#1}}

\newcommand{\Yin}[1]{Y^{(i)}_{#1}}

\newcommand{\bfXi}{\mathbf{X}^{(i)}}

\newcommand{\bA}{\mathbf{A}}

\newcommand{\bfYi}{\mathbf{Y}^{(i)}}

\newcommand{\Reg}{R}

\usepackage[normalem]{ulem}
\DeclareRobustCommand{\erase}{\bgroup\markoverwith{\textcolor{red}{\rule[.5ex]{2pt}{0.4pt}}}\ULon}
\usepackage{comment}

\title{High-dimensional Nonparametric Contextual Bandit Problem}

\author{Shogo Iwazaki$^1$ \and Junpei Komiyama$^{2,4}$ \and Masaaki Imaizumi$^{3,4}$}
\address{$^1$MI-6 Ltd., $^2$New York University, $^3$The University of Tokyo, $^4$RIKEN AIP}

\allowdisplaybreaks
\begin{document}

\maketitle

\begin{abstract}
We consider the kernelized contextual bandit problem with a large feature space. This problem involves $K$ arms, and the goal of the forecaster is to maximize the cumulative rewards through learning the relationship between the contexts and the rewards. It serves as a general framework for various decision-making scenarios, such as personalized online advertising and recommendation systems. Kernelized contextual bandits generalize the linear contextual bandit problem and offers a greater modeling flexibility. Existing methods, when applied to Gaussian kernels, yield a trivial bound of $O(T)$ when we consider $\Omega(\log T)$ feature dimensions. To address this, we introduce stochastic assumptions on the context distribution and show that no-regret learning is achievable even when the number of dimensions grows up to the number of samples. Furthermore, we analyze lenient regret, which allows a per-round regret of at most $\Delta > 0$. We derive the rate of lenient regret in terms of $\Delta$.
\end{abstract}

\section{Introduction}

We consider the multi-armed bandit problem \citep{robbins1952,Lairobbins1985}, which involves $K$ arms (actions), each yielding a reward when selected. In this setting, a learner (or decision-maker) must choose one arm at each round with the goal of maximizing the sum of rewards. To do so, the learner must balance two competing objectives: exploring all arms to refine reward estimates and exploiting the knowledge gained to maximize immediate rewards. This trade-off introduces the fundamental exploration versus exploitation dilemma that the learner must resolve.
In particular, the contextual version of this problem involves pre-action context, and the learner attempts to learn the relationship between the contexts and rewards.
The introduction of such contexts enables us to apply this model to many realistic settings, such as recommendation systems \citep{li2010}, online advertising \citep{tang13}, and personalized treatments \citep{Chakraborty2014}. 

The stochastic contextual bandit problem originated from associative reinforcement learning \citep{Strehl2010} with linear probabilistic concept \citep{abe1999}, where rewards are modeled as noisy linear functions of contexts. The term ``contextual bandit'' was later coined by \citep{LangfordZ07}. 
Subsequent research in linear contextual bandits \citep{li2010,chu11a,abbasi2011improved,DBLP:conf/icml/AgrawalG13,DBLP:conf/aaai/DimakopoulouZAI19,DBLP:conf/aistats/HaoLS20} pushed the frontier of contextual bandits. 

The largest limitation of linear contextual bandit algorithms is that it cannot deal with interactive terms unless handcrafting them into the features. Kernelized contextual bandit algorithms can overcome the limitation of linear models by defining a kernel function $K(X, X')$ that measures the similarity between two features $X$ and $X'$. It learns highly nonlinear patterns without considering all the combinations of interactive terms. 
Existing works of kernelized contextual bandit problems \citep{DBLP:conf/icml/SrinivasKKS10,krause2011contextual,srinivas2012information,valko2013finite} built the problem complexity in terms of the maximum information gain \cite{DBLP:conf/icml/SrinivasKKS10}.
These works introduced algorithms of regret $\tilde{O}(\sqrt{\gamma_T T})$, where $\gamma_T$ is the maximum information gain that depends on the kernel. \citet{DBLP:conf/icml/SrinivasKKS10} further provided the value of $\gamma_T$ for linear, RBF, and Matérn kernels. 
While such an analysis can provide a principled way to build algorithms based on optimism under uncertainty, their regret guarantee is unsatisfying when it comes to popular kernels, such as the Gaussian kernel. For the Gaussian kernel, the maximum information gain is $\gamma_T = O((\log T)^{d+1})$,
which grows much faster than $T$. For example, when $d = \Omega(\log T)$, then $\gamma_T = \Omega(T)$, which results in a trivial linear regret bound. 
This poses a severe limitation on the practical applicability of kernelized contextual bandit algorithms, even for moderately-sized feature spaces. For instance, in a personalized recommendation system with just a hundred features, traditional kernelized contextual bandits would require more than $10^{40}$ samples to have a non-trivial regret bound.

In summary, existing work on contextual bandits has been restricted to linear models or kernel-based approaches that operate in at most $O(\log T)$ dimensions. 
In this paper, we introduce an alternative tool for bounding the regret through the theory of high-dimensional kernel regression. 
When the number of samples used for the estimation is proportional to the number of dimensions, an interpolation estimator, which fits the data perfectly, has a good generalization property through an implicit regularization \citep{el2010spectrum,liang2020just}. 
Building upon this promising property of the interpolation estimator, this paper adapts it to the contextual bandit setting. We propose a conceptually simple explore-then-commit algorithm leveraging this estimator and provide a novel regret analysis framework for high-dimensional kernelized contextual bandits.

The following summarizes our main contributions:
\begin{itemize}

\item The first contribution is on the kernel interpolation estimator. We derive the accuracy of the interpolating estimator in terms of the spectral property of the contexts, and we show three classes of covariates where consistent learning is achieved. Two of them are the sparsity on the context, and the other class assumes the decay rate of the context. These are extracted from \citet{liang2020just} and are stronger than that in the following two senses. (a) While \citet{liang2020just} only analyzes the nonlinear activation of the inner-product class of kernels, we also analyze the RBF class of kernels by using the results in \citet{el2010spectrum} in a more explicit form in terms of convergence rate. These two classes of kernels cover most of the popular kernels as described in Section \ref{subsec_assumption}.
Moreover, (b) unlike \citet{liang2020just}, we allow the confidence level to vary on $d$, which will be discussed in We have a discussion on this in Appendix \ref{sec:detail_theorem}. These results have its own theoretical interest that is independent of the bandit problem.

\item The second contribution is on the introduction of a bandit algorithm with a sublinear regret. The amount of exploration period is carefully chosen so that regret is minimized. For one class of sparse covariates, we show the feasibility of no-regret learning. Moreover, for all three classes of covariates, we show that the regret-per-round can be arbitrarily small and derive a version of lenient regret, which tolerates the forecaster to have an expected regret-per-round $\Delta > 0$.

\item The final contribution is the empirical verification of these theoretical results. We conducted numerical situations to verify the advantage of the proposed method over the existing kernelized bandit algorithms \citep{krause2011contextual,srinivas2012information}.
We also tested linear interpolation algorithms \cite{bartlett2020benign,DBLP:conf/nips/KomiyamaI23} and suggested the value of non-linear modeling in a high-dimensional setting.

\end{itemize}

This paper is organized as follows.
We start with the problem setting (Section \ref{sec_prelim}) and then introduce the explore-then-commit algorithm (Section \ref{sec_etc}). These are followed by the analysis section (Sections \ref{sec_analysis}). The analysis section first discusses the assumptions on the covariates. We then introduce the estimation error of the interpolation estimator. Similarly to low-dimensional models, errors in high-dimensional models are decomposed into the bias term and the variance terms. The sum of these terms characterizes the error of the interpolation estimator. After the evaluation of such an error as a function of the exploration period, we discuss the desired balance between exploration and exploitation in the explore-then-commit algorithm. 
A reader interested in practical performance may skip this section to reach the following numerical simulations (Section \ref{sec_simulation}), which verify the strong performance of EtC with properly tuned exploration period despite of its simplicity.
Finally, our discussion (Section \ref{sec_discuss}) concludes the paper.

\subsection{Related work}\label{subsec_related}

\textbf{Low-dimensional models:}
Research on generalized linear contextual bandits spans several studies \citep{DBLP:conf/icml/LiLZ17,DBLP:journals/corr/abs-2003-10113,DBLP:journals/csysl/RenZK22,DBLP:journals/mor/BlanchetXZ24}. 
 Building on neural-tangent-kernel theory, \citet{DBLP:conf/icml/Zhou0G20,zhang2021neural} introduced NeuralUCB and Neural Thompson sampling. 
For these algorithms to learn consistently, their parameter must be $O(\sqrt{T})$.

\textbf{High-dimensional models in linear bandits:}
Achieving consistency in high-dimensional problems typically demands extra structural assumptions. One widely studied approach is sparsity, where most coefficients are zero; see, e.g., \citet{wang2018minimax,kim2019doubly,bastani2020online,hao2020high,oh2021sparsity,li2022simple,JangZJ22}. Sparse models are consistent when the majority of the coefficients are zero.
Outside the sparse regime, “benign overfitting’’ theory \citep{bartlett2020benign} helps explain why dense, high-dimensional models can still generalize well. When the feature space has a moderate effective rank, unregularized linear regression \citep{bartlett2020benign} and ridge regression \citep{tsigler2020benign} are consistent, regardless of the ambient dimension $d$.
Building on this theory, \citet{DBLP:conf/nips/KomiyamaI23} showed that, in high-dimensional linear contextual bandits, an explore-then-commit strategy is provably consistent when regret is measured with respect to these effective dimensions.

\textbf{High-dimensional models in kernel bandits:} 
By imposing additional structural assumptions for the kernel-based model, several existing works tackle a high-dimensional setting in kernel bandit. \citet{chen2012joint,djolonga2013high} consider the setting where the effective dimension of the underlying reward function is low. \citet{kandasamy2015high} assume that the reward function has an additive form of several independent low-dimensional functions. These existing studies only work under specific assumptions on the reward function. 
These works are somewhat orthogonal to this work. This work imposes assumptions on the feature distribution but is quite flexible in the form of the reward function. Finally, we would like to remark that there are significant body of work on high-dimensional Bayesian optimization (global optimization with kernel modeling) algorithms many of which do not come with theoretical bounds~\citep[e.g.,][]{wang2016bayesian,eriksson2019scalable,letham2020re,eriksson2021high}.

\section{Preliminary}\label{sec_prelim}

\subsection{Notation}
For $z \in \N$, $[z] := \{1,2,\dots,z\}$.
For $x \in \R$, the notation $\lfloor x \rfloor$ here denotes the largest integer that is less than or equal to a scalar $x$. 
For vectors $X,X' \in \R^p$, $\langle X,X'\rangle := X^\top X'$ is an inner-product, $\|X\|_2^2 := \langle X,X \rangle$ is an $\ell 2$-norm.
For a positive-definite matrix $A \in \R^p \times \R^p$, $\|X\|_\bA^2 := \langle X, \bA X \rangle$ is a weighted $\ell 2$-norm.
$\|\bA\|_{\mathrm{op}}$ denotes an operator norm of $\bA$. 
For a square matrix $\bA$, $\lambda_j(\bA)$ is the $j$-th largest eigenvalue of $\bA$.
$1_d := (1,1,...,1)^\top \in \R^d$ is a vector whose elements are $1$.
$O(\cdot), o(\cdot), \Omega(\cdot), \omega(\cdot)$ and $\Theta(\cdot)$ denotes Landau's Big-O, little-o, Big-Omega, little-omega, and Big-Theta notations, respectively.
$\tilde{O}(\cdot),\tilde{\Omega}(\cdot)$, and $\tilde{\Theta}(\cdot)$ ignore polylogarithmic factors.

\subsection{Problem setup}

This paper considers a nonparametric contextual bandit problem with $K$ arms.
We consider the fully stochastic setting, where the contexts, as well as the rewards, are drawn from fixed distributions.
Let $\Omega \subset \R^d$ be a compact domain. 
For each round $t \in [T]$ and arm $i \in [K]$, we define $\Xit$ as a $\Omega$-valued random vector.
We assume $\Xit$ is independent and identically distributed among rounds (i.e., vectors in two different rounds $t,t'$ are independent and identically distributed) but allow vectors $\Xarmt{1},\Xarmt{2},\dots,\Xarmt{K}$ to be correlated with each other.
The forecaster chooses an arm $\Ich \in [K]$ based on the $\Xit$ values of all the arms, and then observes a reward that follows a linear model as shown in\begin{align}
\Yt = f_\ast^{(I(t))}( \Xt) + \xi(t), \label{eq:linear_model}
\end{align}
where $f_\ast^{(i)}: \R^d \to \R$ be an unknown function for  $i \in [K]$.
The independent noise term $\xi(t)$ has zero mean and bounded variance $\Var[\xi(t)] \leq \sigma^2$ with some $\sigma \geq 0$. For the sake of simplicity, we assume that it does not depend on the choice of the arm. However, our results can be extended to the case where $\xi(t)$ varies between arms.

Given the contexts, we define $i^{\ast}(t) \in \argmax_{i \in [K]} f_\ast^{(i)}(\Xit)$ as the (ex ante) optimal arm at round $t$.
Our goal is to design an algorithm that maximizes the total reward, which is equivalent to minimizing the expected regret: $\Reg(T) := \sum_{t=1}^T r(t)$, where:
\begin{equation*}
r(t) = \Ep \left[ f_*^{(i^\ast(t))}( \Xarmt{i^\ast(t)}) \right] - \Ep \left[f_\ast^{(\Ich)}( \Xt) \right]
\label{def:regret}
\end{equation*}
Here, the expectation is taken with respect to the randomness of the history of previous rounds, current contexts, and with (possibly randomized) choice of arm $\Ich$.

In this paper, we focus on the high-dimensional setting where the dimension $d$ is allowed to increase with the total step size $T$ under the moderate size of the number $K = O(1)$ of arms.

\section{Explore-then-Commit with Kernel Interpolation}
\label{sec_etc}

The \textit{Explore-then-Commit} (EtC) algorithm is a well-known approach for solving high-dimensional linear bandit problems, and it has been shown to be effective in previous studies such as \citet{hao2020high,li2022simple,DBLP:conf/nips/KomiyamaI23}. The EtC algorithm operates by first conducting $T_0 = NK < T$ rounds of exploration, during which it uniformly explores all available arms to construct an estimator $\hat{f}^{(i)}$ for each arm $i \in [K]$. After the exploration phase, the algorithm proceeds with exploitation.
Let $N$ be the number of the draws of arm $i$, and $\bfXi = (\Xin{1},...,\Xin{N})^\top \in \R^{N \times d}$ and $\bfYi = (\Yin{1},...,\Yin{N})^\top \in \R^{N}$ be the observed contexts and rewards of arm $i$, where $(\Xin{n},\Yin{n})$ is the corresponding values on the $n$-th draw of arm $i$.
Since we choose $\Ich$ uniformly during the exploration phase, these are independent and identically drawn samples.

\begin{algorithm}[htbp]
\caption{Explore-then-Commit (EtC)}\label{alg:etc} 
\begin{algorithmic}
\REQUIRE Exploration duration $T_0$.
\FOR{$t=1,.., T_0$}
    \STATE Observe $\Xit$ for all $i \in [K]$.
    \STATE Choose $\Ich = t - K\lfloor t/K \rfloor$. 
    \STATE Receive a reward $\Yt$.
    
\ENDFOR
\FOR{$i \in [K]$}
    \STATE Calculate $\hat{f}^{(i)}(\cdot)$ as \eqref{def:interp}
\ENDFOR
\FOR{$t = T_0 + 1,...,T$}
    \STATE Observe $\Xit$ for all $i \in [K]$.
    \STATE Choose $\Ich \in \argmax_{i \in [K]}  \hat{f}^{(i)}(\Xit)$.
    \STATE Receive a reward $\Yt$.
\ENDFOR
\end{algorithmic}
\end{algorithm}

For estimating the parameter $f_*^{(i)}$, we consider a kernel regression approach by \citet{liang2020just}.
Let $K:\Omega \times \Omega \to \R$ be a positive definite kernel function, and $\mH$ be a corresponding reproducing kernel Hilbert space (RKHS).
We also define a kernel matrix $K(\bfXi, \bfXi) \in \R^{N \times N}$ with its $(j,j')$-th value as $K(X_j^{(i)}, X_{j'}^{(i)})$.
Then, we consider a minimum-norm interpolator, which perfectly fits the observed data, as
\begin{align}
    &\hat{f}^{(i)} = \argmin_{f \in \mH} \| f\|_\mH, \label{def:interp}\\
    &\mathrm{~s.t.~} \forall j \in [N],~f(X_j^{(i)}) = Y_j^{(i)}. \notag 
\end{align}
Using the interpolator \eqref{def:interp} as an estimator, we present our EtC algorithm based on this estimator in \algoref{alg:etc}.

The interpolator has an explicit form as $\hat{f}^{(i)}(\cdot) =  K(\cdot, \bfXi) K(\bfXi, \bfXi)^{-1} \bfYi$.
The invertibility of $K(\bfXi, \bfXi)$ is always satisfied as long as we employ the class of positive definite kernels $K(\cdot, \cdot)$ (see Section 4 in \cite{christmann2008support} for introduction).
The computational cost of the inverse matrix when using this method can be reduced using methods such as the Nystr\"om approximation \citep{drineas2005nystrom}.
Also, the optimization in \eqref{def:interp} is solved by a convex minimization by coefficients (e.g., see Section 9 of \cite{scholkopf2002learning}), so it is easily computed using algorithms such as the gradient descent algorithm, which takes $O(N)$ computational complexity.

\section{Theoretical Analysis}\label{sec_analysis}

This section analyzes an estimator error and a regret in terms of $d, T$ and covariance $\Sigma_d^{(i)}$. Specifically, we provide some assumptions, then study the estimation error and the regret.

\subsection{Assumption and Kernels}\label{subsec_assumption}

\begin{assumption}[Basic condition] \label{asmp:basic}
For each arm $i\in[K]$, the following hold:
\begin{enumerate}
  \setlength{\parskip}{0cm}
  \setlength{\itemsep}{0cm}
    \item[(i)] A reward function $f_\ast^{(i)}$ is an element of RKHS $\mH$ and satisfies norm $\|f_\ast^{(i)}\|_\mH \leq B < \infty$ with some $B$. 
    \item[(ii)] The covariance matrix $\Sigma_d^{(i)} \coloneqq \E[X_1^{(i)}X_1^{(i)\top}]$ of $X_1^{(i)}$ satisfies $\|\Sigma_d^{(i)}\|_{\mathrm{op}} \leq 1$.
    \item[(iii)] The $\R^{d}$-valued random vector $Z^{(i)}$ is independent among their coordinates, where $Z^{(i)} = (\Sigma_d^{(i)})^{-1/2}X_1^{(i)}$. 
        Furthermore, assume that each element of $Z^{(i)}$ has zero mean, unit variance, and there exists absolute constant $C > 0$ such that $\max_{j \in [d]}|Z_j^{(i)}| \leq C $ holds almost surely.
\end{enumerate}
\end{assumption}
This assumption ensures (i) the reward functions are sufficiently smooth as commonly assumed in the kernel bandit studies \citep{srinivas2012information,chowdhury2017kernelized}, (ii) the raw contexts have controlled spectral magnitude, and (iii) after whitening the contexts, one obtains independent, bounded coordinates with unit variance—properties that jointly facilitate concentration and spectral analyses, which is widely used in the dense high-dimensional settings \citep{bartlett2020benign,tsigler2020benign,liang2020just}.

We consider two families of kernels that we can use in our analysis following \citet{liang2020just,el2010spectrum}.
These families include most of the commonly used kernel functions.
\begin{definition}[Class of kernels] \label{def:kernel_class}
We define the following families of a kernel function $K$ with normalization:  for any $x \in \Omega$, $K(x, x) \leq 1$.
    \begin{enumerate}
  \setlength{\parskip}{0cm}
  \setlength{\itemsep}{0cm}
      \setlength{\leftskip}{-0.7cm}
        \item \textbf{Inner-product class}: With $h \in C^3$ satisfying $h(0)>0$, $h'(0)>0$, and $\min_{a\in[0,1]}h''(a)\ge \underline{c}$ with some $\underline{c}$, the class of kernels is defined as 
   $$
     K(x,x') = h\Bigl(\frac{\langle x,x'\rangle}{d}\Bigr).
   $$
        This class includes, for example,  (i) a linear kernel $h(t)=t$, (ii) all polynomial kernels $h(t)=(t+c)^p, c\in \R, p \in \N$, and (ii) more general Schoenberg kernels $h(t) = \sum_{\ell=0}^\infty a_\ell C_\ell^{((d-2) / {2})}(t)$ on the sphere with the Gegenbauer polynomial $C_\ell^{(\alpha)}(\cdot)$ and coefficients $a_\ell\ge0$.
        \item \textbf{Radial-basis function (RBF) Class}: With $h \in C^3$ satisfies $h(x) > 0, h'(x) < 0, \forall x \in [0, 2]$, and $h'(\cdot)$ is monotonically increasing, the class of kernels is defined as
   $$
     K(x,x') = h\Bigl(\frac{\|x-x'\|^2}{d}\Bigr).
   $$
    This class encompasses (i) the Gaussian kernel $h(t)=\exp(-\gamma t)$,  $\gamma > 0$, (ii) the Laplace kernel $h(t)=\exp(-\gamma\sqrt{t})$, (iii) the rational quadratic kernel $h(t) = (1 + {t} / {2\alpha \gamma^2})^{-\alpha}$, $\alpha > 0$, (iv) the Mat\'ern kernel $h(t) = {2^{1-\nu}}{\Gamma(\nu)^{-1}} ({\sqrt{2\nu}\sqrt{t}} / {\ell} )^\nu
    K_\nu ({\sqrt{2\nu}\sqrt{t}}/  {\ell}), \ell, \nu \geq 2$, and many other kernels.
    \end{enumerate}
\end{definition}

\begin{remark}{\rm (Necessity of Scaling in  kernels)}
Here, scaling $1/d$ in kernel $h(\cdot/d)$ is important for both kernels because the inner-product or norm of a \(d\)-dimensional vector \(d\) grows linearly to $d$. By dividing by \(d\), the input into \(h\) is kept normalized.
This scaling method is widely used in high-dimensional research, as seen in \citet{el2010spectrum} and differs from existing studies on kernelized contextual bandit algorithms, which assume a fixed scaling factor \citep{srinivas2012information,vakili2021information}.
In Appendix~\ref{sec:summary_kb}, we confirm that existing information gain-based analysis with a fixed scaling factor results in a trivial $O(T)$ regret in our setup.
\end{remark}

For a dense high-dimensional setting, the properties of the covariance matrices $\Sigma_d^{(i)}$ are crucial, and they matter the rate of convergence. Existing literature, as well as this paper, assumes that the covariance matrices are functions of the size of the instance. Namely, we consider the sequence of problem instances $(d, (\Sigma_d^{(i)})_{i \in [K]}, T)$ where $d,T$ are nondecreasing. Our primary focus is on whether the problem achieves no regret. Namely, the regret is $o(T)$ with respect to the sequence.

\subsection{Estimation error analysis}\label{subsec_esterror}

We study the estimation error of the function $\hat{f}$ by the kernel method (Eq.~\eqref{def:interp}), utilizing the analysis of minimum-norm interpolating estimator \citep{liang2020just}. 
Specifically, we study the error bounds of $\hat{f}^{(i)}$ in the following norm
\begin{align}
    \|\hat{f}^{(i)} - f_\ast^{(i)}\|_{L_i^2}^2 = \int_{\Omega} ({\hat{f}^{(i)}(x) - f_\ast^{(i)}(x)})^2 P_{X_1^{(i)}}(\mathrm{d}x).
\end{align}
For brevity, we write the kernel gram matrix as $K_X^{(i)} = K(\bfXi, \bfXi)$.
Define $\hat{\Sigma}_d^{(i)} := {\bX^{(i)}\bX^{(i)\top}} / {d}$.

\begin{definition}[kernel parameter sequence] \label{def:kernel_parameters}
We define positive sequences $(\alpha_d^{(i)})_{d}$, $(\beta_d^{(i)})_{d}$, and $(\gamma_d^{(i)})_{d}$ with $\tau_d^{(i)} \coloneqq 2\mathrm{tr}(\Sigma_d^{(i)}) / d$ as follows:
\begin{enumerate}
  \setlength{\parskip}{0cm} 
  \setlength{\itemsep}{0cm}
    \item[(i)] the inner-product class case: $\alpha_d^{(i)} = h(0) + h''(0) \mathrm{Tr}({(\Sigma_d^{(i)})^2})/d^2$, $\beta_d^{(i)} = h'(0) \coloneqq \beta$, and $\gamma_d^{(i)} = h(\tau_d^{(i)}/2) - h(0) - h'(0)\tau_d^{(i)}/2$.
    \item[(ii)] the RBF class case: $\alpha_d^{(i)} = h(\tau_d^{(i)})  + 2h''(\tau_d^{(i)})\mathrm{Tr}({(\Sigma_d^{(i)})^2})/d^2, 
            \beta_d^{(i)} = -2h'(\tau_d^{(i)})$, and $
                \gamma_d^{(i)} = h(0) + \tau_d^{(i)} h'(\tau_d^{(i)}) - h(\tau_d^{(i)}).$
\end{enumerate}
\end{definition}
Intuitively, this sequence corresponds to the coefficients obtained when performing a Taylor expansion of the Gram matrix, composed of each kernel, around the empirical covariance matrix. A more detailed derivation is provided in Appendix~\ref{sec:detail_theorem}.

In the following, we introduce two quantities that characterize the error. Here, we assume the error of the interpolation estimator with $N$ data points. In EtC (Section \ref{sec_etc}), we uniformly explore over the $K$ arms, and thus the number of exploration rounds corresponds to $NK$.
\begin{definition}[Effective bias/variance for kernel gram matrix]
Let $(\alpha_d^{(i)})$, $(\beta_d^{(i)})$, and $(\gamma_d^{(i)})$ be the sequences in Definition \ref{def:kernel_parameters}.
Then, we define an effective variance $\mathcal{V}_{d,N}^{(i)}$ and bias $\mathcal{B}_{d,N}^{(i)}$ of the estimator $\hat{f}^{(i)}$  follows:
    \begin{align}
        &\mathcal{V}_{d,N}^{(i)} \coloneqq  d^{-1} \sum_{j=1}^N \frac{\lambda_j\rbr{\hat{\Sigma}_d^{(i)}}}{\sbr{{\gamma_d^{(i)}}/{\beta_d^{(i)}} + \lambda_j\rbr{\hat{\Sigma}_d^{(i)}} }^2}, \mbox{~~~and~~~}
        \mathcal{B}_{d,N}^{(i)} \coloneqq \inf_{0 \leq k \leq N}\sbr{\frac{1}{N} \sum_{j > k} \lambda_j (K_X^{(i)}) + 2 \sqrt{\frac{k}{N}}}.
    \end{align}
\end{definition}
Here, $\mathcal{V}_{d,N}^{(i)}$ corresponds to the variance of the estimation error, and $\mathcal{B}_{d,N}^{(i)}$ corresponds to the bias. $\mathcal{V}_{d,N}^{(i)}$ is expressed as the sum of the ratios of the eigenvalues $\lambda_j(\hat{\Sigma}_d^{(i)})$ of the Gram matrix $\hat{\Sigma}_d^{(i)}$, which corresponds to the variance component in the ordinary ridge regression~\citep[e.g.,][]{tsigler2020benign}. $\mathcal{B}_{d,N}^{(i)}$ is represented as the sum of eigenvalues $\lambda_j (K_X^{(i)}), j > k$ above a certain threshold $k$, which explains the magnitude of the components that the model ignores during estimation.

The following theorem bounds the error in terms of the bias and the variance above. 
\begin{theorem}[Estimation error] \label{thm:inner-product}
Consider a kernel $K$ from Definition \ref{def:kernel_class}.
Fix any $i \in [K]$.
Suppose $c_L \leq d/N \leq c_U$ holds with some universal constants $c_L, c_U \in (0, \infty)$.
Suppose that, in the case of the RBF class, $\|X_1^{(i)}\| = o(d^{-c})$ and $\gamma_d^{(i)} \geq \underline{c} \mathrm{Tr}(\Sigma_d^{(i)})/d$ hold for sufficiently large $d$ and  some $c > 0$.
Then, under Assumption \ref{asmp:basic}, with probability at least $1 - O(d^{-\nu})$ with some $\nu \in (0,1)$, the following inequality holds with some constant $C>0$ and sufficiently large $d$:
    \begin{equation}\label{ineq_errorbound}
    \begin{split}
                &\Ep [\|\hat{f}^{(i)} - f_\ast^{(i)}\|_{L_i^2}^2 \mid \bX^{(i)}]  \leq C \left( \mathcal{V}_{d,N}^{(i)} + \mathcal{B}_{d,N}^{(i)} \right) + \tilde{O}(d^{-1} + N^{-1/2}).
    \end{split}
    \end{equation}
\end{theorem}
Here, $\nu \in (0,1)$ is an existing constant depending on $\gamma_d^{(i)}, \tau_d^{(i)}$, and $c$ in Definition \ref{def:kernel_class}.
A detailed description will be provided in Section \ref{sec:detail_theorem}.

The magnitude of these two terms depends on the spectral property of the covariance matrices. In the following, we present several examples of the benign covariance and kernel gram matrix in which the bias and variances are explicitly controlled. 
\begin{example}[Setups: Section 4 in \citet{liang2020just}]
\label{ex:liang_sec4}

Here, with a constant $\delta > 0$, we consider a limit $d,N \to \infty$ while $c_L \delta \leq  d/N \leq c_U \delta$ holds. 
Then, we have the following situations with some constant $C>0$ and with probability $1-o(1)$ as $d,N \to \infty$:
\begin{enumerate}
  \setlength{\parskip}{0cm} 
  \setlength{\itemsep}{0cm} 
    \item[Case I.] ($c_U \delta < 1$ and low-rank case): $\Sigma_d^{(i)}$ has a form $\mathrm{diag}(1,...,1,0,...,0)$ with its rank $\varepsilon d$ with $\varepsilon \in (0, 1]$. Then, we have $\mathcal{V}_{d,N}^{(i)} \leq C\delta \varepsilon$ and $\mathcal{B}_{d,N}^{(i)} \leq C \varepsilon$.
    \item[Case II.] ($c_U \delta < 1$ and approximately low-rank case): $\Sigma_d^{(i)}$ has a form $\mathrm{diag}(1,\varepsilon,...,\varepsilon)$ with $\varepsilon \asymp \delta^{1/2}$. Then, we have $\max\{\mathcal{V}_{d,N}^{(i)}, \mathcal{B}_{d,N}^{(i)} \} \leq C \delta^{1/2}$.
    \item[Case III.] ($c_L \delta > 1$ and spectral decay): $\trace(\Sigma_d^{(i)})$ satisfies $\varepsilon \coloneqq \trace(\Sigma_d^{(i)})/d \leq C\delta^{-1/3}$. Then, we have $\max\{\mathcal{V}_{d,N}^{(i)}, \mathcal{B}_{d,N}^{(i)} \} \leq  C \delta^{-1/3}$.
\end{enumerate}
\end{example}

\begin{remark}
Theorem \ref{thm:inner-product} and Example \ref{ex:liang_sec4} establish bounds under the assumption that $d/N \sim \delta$--that is, the ratio between the dimension $d$ and the number of data points $N$ remains fixed but is tuned to balance the bias and the variance. The analysis for these results leverages techniques from random matrix theory, which states that the limiting behavior of the eigenvalues follows the Marchenko-Pastur distribution~\cite{el2010spectrum} under the assumption that $d/N$ remains fixed. 
Note that Eq.~\eqref{ineq_errorbound}, the residual term of $\tilde{O}(d^{-1} + N^{-1/2})$ vanishes as we have $d, N \rightarrow \infty$, and the residual term is minor.

\end{remark}

\subsection{No-regret guarantee under vanishing generalization error regime}
Our first result is the following \thmref{thm:etc_reg_simple}, which 
shows that, for Case I in \exref{ex:liang_sec4}, EtC is the no-regret algorithm with the proper choice of the exploration duration $T_0$ even in the regime of $d, T \rightarrow \infty$.

\begin{theorem}[No-regret EtC for inner-product class]
\label{thm:etc_reg_simple}
Consider a kernel from the inner-product class as Definition \ref{def:kernel_class}.
    Suppose that $d = \Theta(T^{\tau_1})$ and Assumption~\ref{asmp:basic} holds, 
    where $\tau_1 \in (0, 1)$.
    Furthermore, assume the contexts are generated through low-rank covariance structure (Case I in \exref{ex:liang_sec4}) with $\varepsilon = \Theta(T^{-\tau_2})$, 
    where $\tau_2 \in (0, \tau_1 / 4)$. 
    Then, with $\overline{\tau} = \min\cbr{{\tau_1}/  {2}, \tau_2, \tau_1 - 4\tau_2} > 0$, EtC with the exploration duration $T_0 = \Theta(Kd)$ 
    is a no-regret algorithm, and the regret satisfies
    \begin{equation}
        R(T) = \tilde{O}\rbr{T^{\max\{\tau_1, 1 - {\overline{\tau}}/  {2}\}}}.
    \end{equation}
\end{theorem}
\begin{theorem}[No-regret EtC for RBF class]
\label{thm:etc_reg_simple_rbf}
    Consider a kernel from the RBF class as Definition \ref{def:kernel_class}.
    Suppose that $d = \Theta(T^{\tau_1})$ and Assumption~\ref{asmp:basic} holds, 
    where $\tau_1 \in (0, 1)$.
    Also, suppose that $\|X_1^{(i)}\| = o(d^{-c})$ and $\gamma_d^{(i)} \geq \underline{c} \mathrm{Tr}(\Sigma_d^{(i)})/d$ hold for sufficiently large $d$ and  some $c > 0$.
    Furthermore, assume the contexts are generated through low-rank covariance structure (Case I in \exref{ex:liang_sec4}) with $\varepsilon = \Theta(T^{-\tau_2})$, 
    where $\tau_2 \in (0, \tau_1 / 2)$. 
    Then, with $\overline{\tau} = \min\cbr{{\tau_1}/  {2}, \tau_2, \tau_1 - 2\tau_2} > 0$, EtC with the exploration duration $T_0 = \Theta(Kd)$ 
    is a no-regret algorithm, and the regret satisfies
    \begin{equation}
        R(T) = \tilde{O}\rbr{T^{\max\cbr{\tau_1, 1-2\tau_1, 1- c\tau_1, 1 - {\overline{\tau}} / {2}}}}.
    \end{equation}
\end{theorem}
\begin{proof}[Proof sketch of Theorem~\ref{thm:etc_reg_simple}]
We set the exploration period $T_0 = \Theta(K d)$ so that $d/N = d/(T_0/K)$ remains constant. The regret is bounded by the sum of 
the regret during the exploration period and that during exploitation period. By Lemma \ref{lem:reg_upper}, we decompose the regret-per-round during the latter period into as follows:
Firstly, When Eq.~\eqref{ineq_errorbound} holds, we can bound the regret-per-round in terms of bias and variance. 
Secondly, with a small probability, Eq.~\eqref{ineq_errorbound} does not hold. In this case, we bound the regret-per-round by a constant.
\end{proof}

\begin{remark}{\rm (Limitation in the regret bound)}
The No-regret guarantee provided in \thmref{thm:etc_reg_simple} is only valid under Case I in \exref{ex:liang_sec4}. The difficulty of achieving the no-regret guarantees in the high-dimensional kernel regression~\cite{el2010spectrum,liang2020just} stems from the limitation that it must keep $d/N$, the ratio between the dimension and the number of training sample, to be a constant while $d,N \rightarrow \infty$. For the other two cases in Example \ref{ex:liang_sec4}, no-regret requires us to adapt this ratio as a function of $d$, which does not fit into the current framework. 

\end{remark}

In the next subsection, we discuss a slightly weaker version of regret that is achievable for all three cases in Example \ref{ex:liang_sec4} under some assumptions.

\subsection{Lenient regret under non-vanishing generalization error regime}
\label{sec:lr_regret}
In this subsection, we consider the setting where implicit regularization $\varepsilon$ is fixed over time steps but is sufficiently small. 
We focus on the following weaker version of regret $R_{\Delta}(T)$ in this section:
\begin{equation}
    R_{\Delta}(T) = \sum_{t=1}^T \Phi_{\Delta}\rbr{r(t)} ,
\end{equation}
where $\Phi_{\Delta}(a) = \max\{a - \Delta, 0\}$, and $\Delta > 0$ is the hyperparameter specified before the algorithm run. 
This definition of regret is motivated by \emph{lenient regret}, 
which is originally introduced in multi-armed bandits~\citep{merlis2021lenient} and extended to kernel bandits~\citep{cai2021lenient}.
The lenient regret diverges slowly if the implicit regularization $\varepsilon$ is sufficiently small and error terms become negligible 
compared to the gap $\Delta$.
The following \thmref{thm:etc_lr_Td} state this intuition formally by showing that the lenient regret increases sub-linearly with the proper choice of $T_0$ in the inner-product class of kernels.
The result for the RBF class of kernels is shown in Appendix~\ref{subsec:lr_rbf} due to the space limitation.

\begin{theorem}[Lenient regret bound for inner-product class]
    \label{thm:etc_lr_Td}
    Fix any $\Delta > 0$ as a constant. Consider the setup of Theorem \ref{thm:inner-product} with the cases $\mathrm{I}$, $\mathrm{II}$, and $\mathrm{III}$ as Example \ref{ex:liang_sec4}.
    Define $h_{max}'' = \max_{a \in [0, 1]} h''(a) \in (0, \infty)$ and $h_{min}'' = \min_{a \in [0, 1]} h''(a) \in (0, \infty)$, and 
    assume that $d = \Theta(T^{\tau})$ holds for some $\tau \in (0, 1)$. 
    Furthermore, each of the cases $\mathrm{I}$, $\mathrm{II}$, and $\mathrm{III}$, if $\varepsilon$ is sufficiently small as a function of $\Delta$, then we can choose $T_0 = O(d)$ and
    the lenient regret for each of the three examples in Example \ref{ex:liang_sec4} is $R_{\Delta}(T) = O(d) = o(T)$.
 \end{theorem}

\begin{proof}[Proof sketch of Theorem \ref{thm:etc_lr_Td}]
To achieve no-lenient-regret, it suffices to choose the amount of exploration $T_0 = NK$ such that $\mathcal{V}_{d,N}^{(i)} +\mathcal{B}_{d,N}^{(i)} < \Delta$
with a high probability. For the three scenarios in Example \ref{ex:liang_sec4}, we achieve this by fixing the value of $\delta = d/N$ appropriately so that the bias and variance terms are balanced. For sufficiently large $d,N$, the residual term in Eq.~\eqref{ineq_errorbound} is small enough compared with $\Delta$, and it suffices to bound the lenient regret.
\end{proof}

\begin{remark}{\rm (Comparison between no-regret and no-lenient-regret)}\label{rem_comparison_le}
For the latter two of the three examples in Example \ref{ex:liang_sec4}, we can achieve a zero lenient regret bound for any $\Delta > 0$. However, we found that a no-regret guarantee is much more challenging. This might seem contradictory at first because, if we can achieve any small error we want, it seems that we are able to achieve perfect learning. The technical reason is that, no-regret corresponds to gradually decrease the value of $\Delta$ of lenient regret to zero as $d$ grows. This would require us to change the ratio $d/N$ as a function of $T$. However, our framework (Theorem \ref{thm:inner-product}) specifically requires this ratio to be a constant. In contrast, for zero lenient regret, choosing a fixed $\delta$ as a function of $\Delta$ suffices. In this case, we can keep the ratio $d/N$ constant.
\end{remark}

\begin{figure*}
    \centering
    \includegraphics[width=0.8\linewidth]{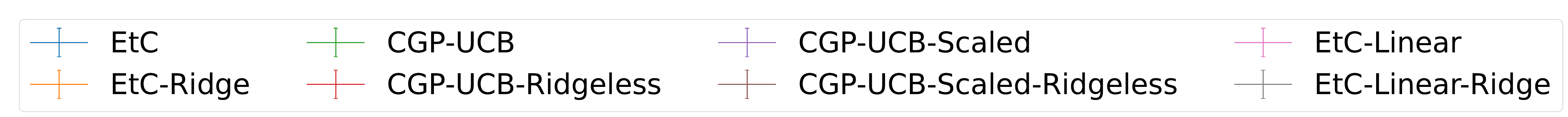} \\
    \includegraphics[width=0.3\linewidth]{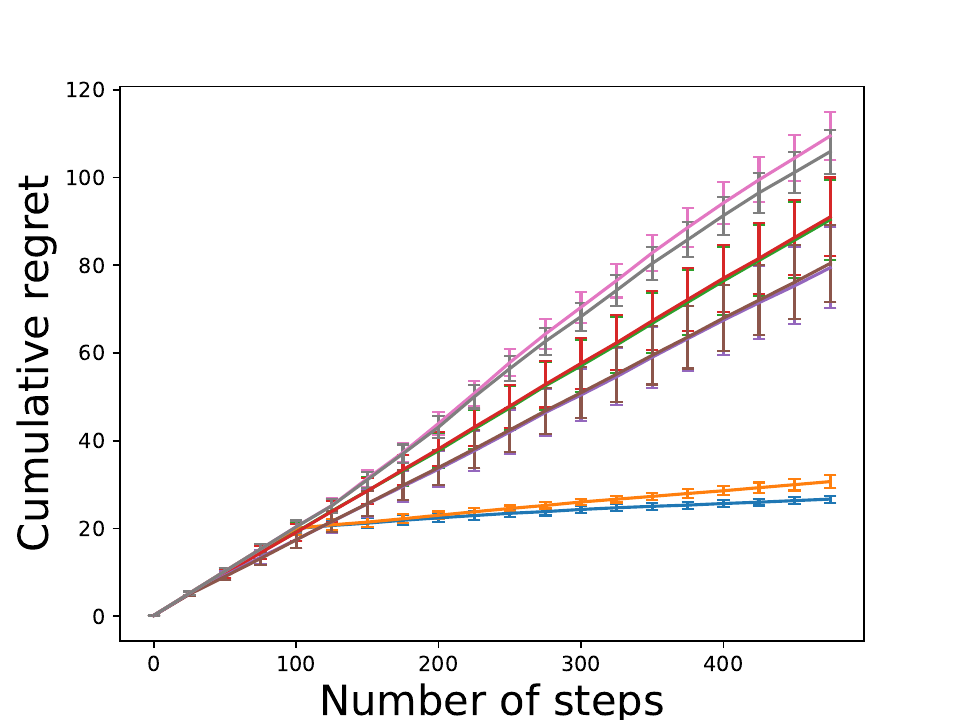}
    \includegraphics[width=0.3\linewidth]{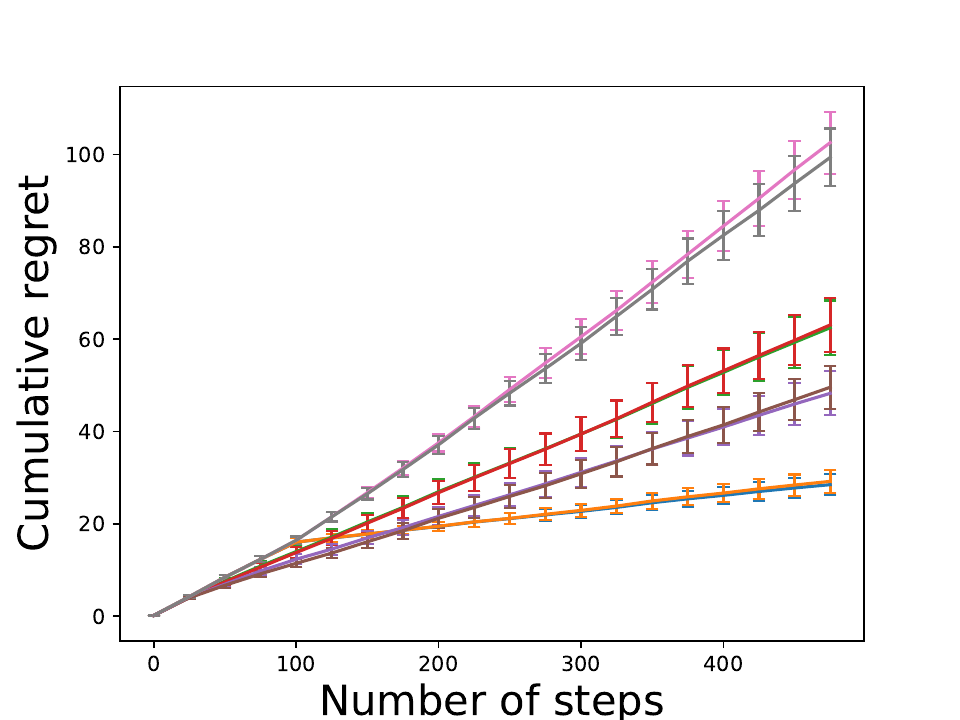}
    \includegraphics[width=0.3\linewidth]{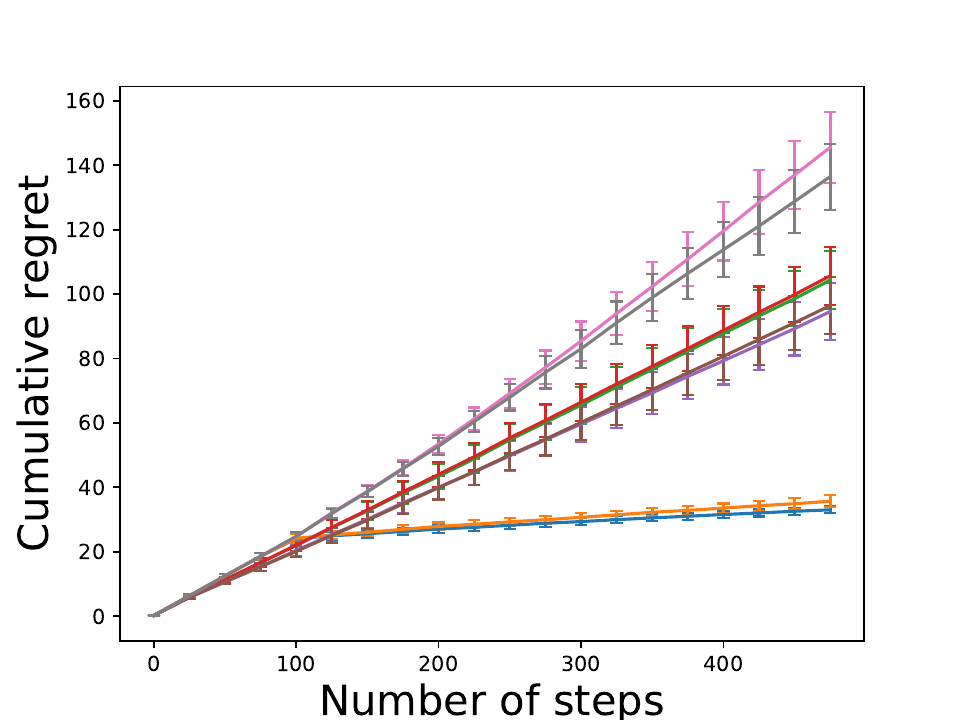}
    \\
    \includegraphics[width=0.3\linewidth]{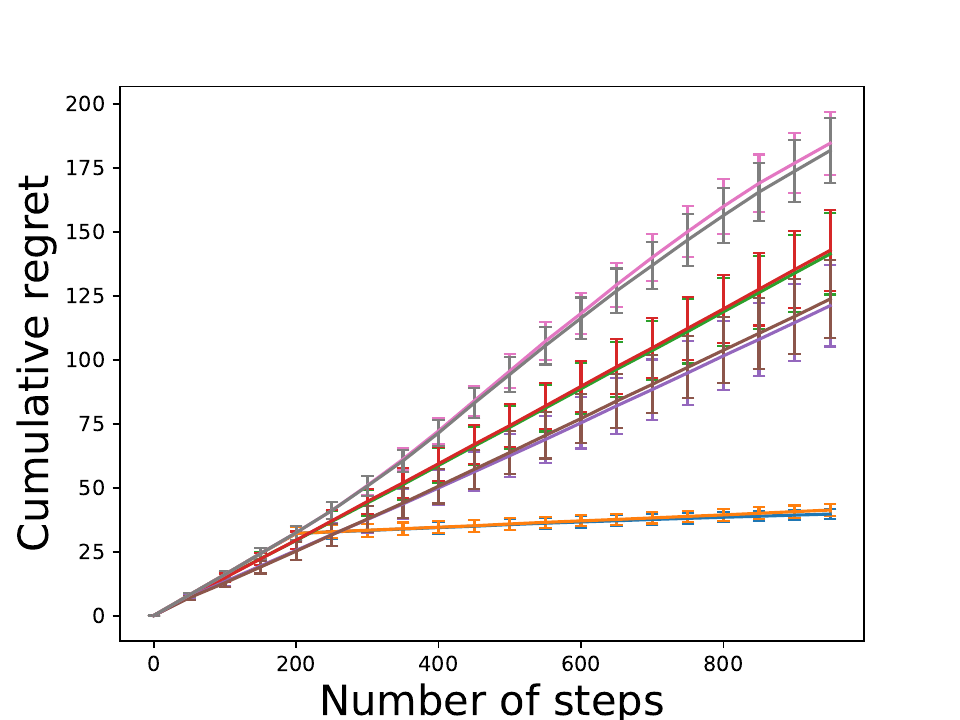}
    \includegraphics[width=0.3\linewidth]{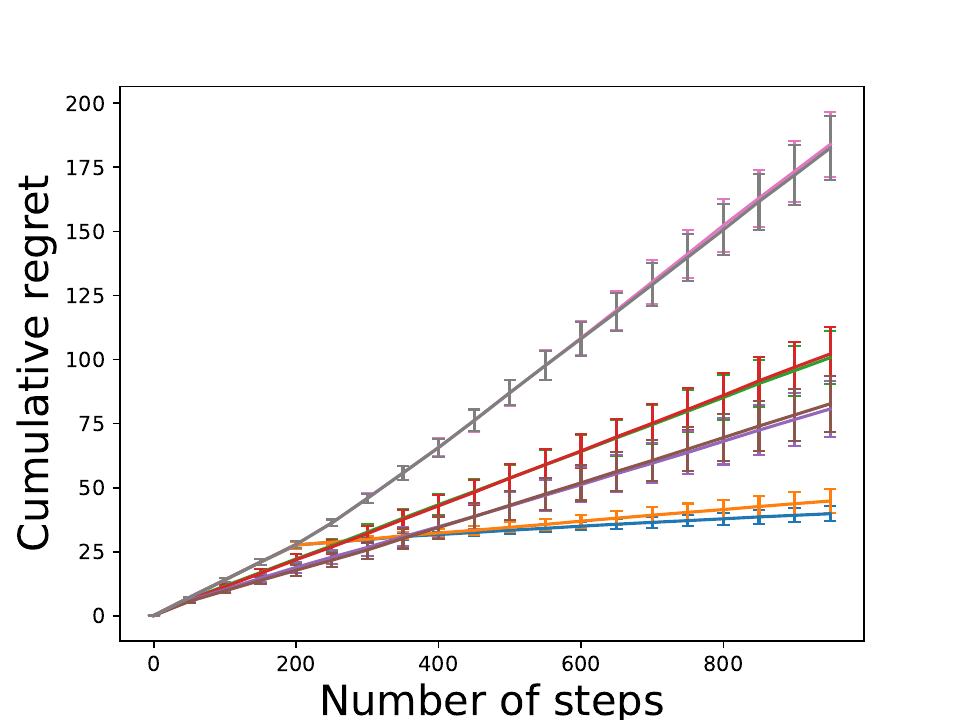}
    \includegraphics[width=0.3\linewidth]{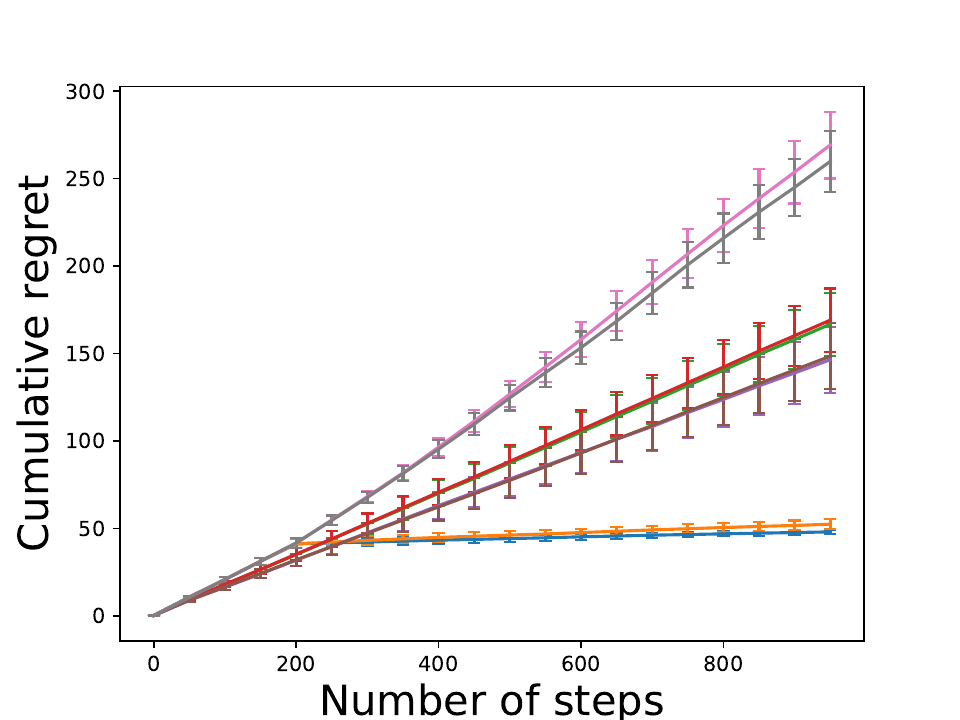}
    \caption{The average cumulative regret with different $10$ seeds. The error bars represent one standard error. 
    The left, middle, and right figures show the results in low-rank, approximate low-rank, and spectral decay covariance matrix settings, respectively. The top and bottom figures show the results in $(d, T_0) = (100, 100)$ and $(d, T_0) = (200, 200)$, respectively.}
    \label{fig:emperical_result}
\end{figure*}

\section{Simulation Study}\label{sec_simulation}

\paragraph{Baselines:} We consider the following baselines:
(i) \textbf{CGP-UCB/CGP-UCB-Ridgeless} as  contextual variants of GP-UCB~\citep{krause2011contextual}, (ii) \textbf{CGP-UCB-Scaled/CGP-UCB-Scaled-Ridgeless} as variants of the previously introduced CGP-UCB algorithm, (iii) \textbf{EtC-Linear/EtC-Linear-Ridge} as variants of the
    EtC algorithm that adopts a linear kernel~\citep{DBLP:conf/nips/KomiyamaI23}.
Also, (iv) \textbf{EtC/EtC-Ridge} are our proposed EtC algorithms.
Details of the implementation of each algorithm are described in Section \ref{sec:details_experiments}.

\paragraph{Data Generation:}
We generate the reward in our experiments from the reward function $f_{\ast}^{(i)}(\cdot) = \sum_{m=1}^{500} c_m^{(i)} F_{\mathrm{RBF}}(\cdot, x_m^{(i)})$, with $F_{\mathrm{RBF}}(x, \tilde{x}) = \exp(-{\|x - \tilde{x}\|_2^2} / {0.25d})$ and $c_m^{(i)} \sim \mathrm{Uniform}([-1, 1])$.
We set $K = 20$ and $\xi(t) \sim \mN(0, \sigma^2)$ with $\sigma^2 = 10^{-4}$. 
Also, the context of the $i$-th arm $X_1^{(i)}$ is generated in each of three different settings of the covariance matrix $ \Sigma_j^{(i)}$ as follows: (i) a \textbf{low-rank} case in which  the covariance matrix has an exact low-rank structure, (ii) an \textbf{approximate low-rank} case in which  the covariance matrix is approximately low-rank with some randomness, and (iii) a \textbf{spectral decay} case in which the eigenvalues of the covariance matrix have a specific decay rate. 
More details are described in Section \ref{sec:details_experiments}.

\paragraph{Results:}  Figure~\ref{fig:emperical_result} shows the result. The followings discuss our major findings. 
Firstly, our kernel-based EtC algorithms consistently outperform the existing of CGP-UCBs and EtC-Linear methods. 
Meanwhile, we also observe that a ridge-based (regularized) EtC method,
which is not covered by our theoretical analysis, has competitive empirical performance. 
Secondly, linear-based EtC algorithms have the worst performance. 
This is because these algorithms tend to misspecify the optimal arms after the exploration period 
due to limited model flexibility. 
At a high level, these findings suggest that methods designed for low-dimensional scenarios do not always yield strong empirical results in high-dimensional settings, even under additional assumptions such as those described in Example~1.
Consequently, it is essential to design algorithms suited for high-dimensional problems that exploit the underlying problem settings.

\section{Real-World Application}

We provide the experimental results with the real-world situation based on the Avazu CTR prediction data \citep{zhu2021open}.
This real-world application covers a recommendation system of the online advertisement. Specifically, in online advertising, a platformer desires to optimize the ad selection based on the non-linear interaction between the user and the visited website features. 
Since they often include an enormous number of categorical features (such as demographic properties of the user), it easily leads the problem into high-dimensional settings. We provide a detailed setup in \secref{sec:avazu_detail}.

Table \ref{tab:real_world} presents the results. Here, both ETC (the proposed method) and ETC-Linear, designed to handle high-dimensions, achieve strong performance regardless of whether ridge regularization is applied. In particular, ETC and ETC-Ridge leverages nonlinearity to attain lower regret.

\begin{table}[htbp]
\centering
\caption{Final cumulative regret with 5 seeds.}
\begin{tabular}{lcccccccc}
\hline
\textbf{Method} & \multicolumn{2}{c}{ETC (proposed)} & \multicolumn{2}{c}{CGP-UCB} & \multicolumn{2}{c}{CGP-UCB-Scaled} & \multicolumn{2}{c}{ETC-Linear} \\
 & vanilla & Ridge & vanilla & Ridgeless & vanilla & Ridgeless & vanilla & Ridge \\
\hline
\textbf{Regret} & \underline{524.8} & \textbf{520.8} & 665.1 & 665.0 & 669.6 & 650.6 & 562.1 & 554.3 \\
\textbf{(±SE)} & (±3.4) & (±1.7) & (±0.3) & (±1.0) & (±2.4) & (±0.8) & (±1.2) & (±2.9) \\
\hline
\end{tabular}
\label{tab:real_world}
\end{table}

\section{Conclusion}\label{sec_discuss}

We studied the high-dimensional kernelized contextual bandit problem by extending the recent generalization 
error bound of kernel ridgeless regression to the decision-making problem. 
Specifically, we showed that if the spectral property of the context is favorable such that the generalization error converges to zero, the simple EtC algorithm achieves a sub-linear increase of regret without additional structural assumption of the reward. Furthermore, even in harder problem settings with a non-vanishing generalization error, we provide the vanishing lenient regret guarantees with careful tuning of the exploration duration of EtC.
We also conducted numerical simulations. Their results are promising. In many scenarios, the proposed EtC algorithm outperforms existing kernelized UCB algorithms that often overestimate the amount of exploration in practical scenarios. These simulations also verified the value of non-linearity by demonstrating the limited performance of linear interpolators.

\appendix

\section{Details of Experiments} \label{sec:details_experiments}

\subsection{Algorithms}

\begin{itemize}
  \setlength{\parskip}{0cm} 
  \setlength{\itemsep}{0cm} 
    \item \textbf{CGP-UCB/CGP-UCB-Ridgeless} are contextual variants of GP-UCB~\citep{krause2011contextual}. We set the regularization parameter $\lambda^2$ of kernel-ridge regressor (or noise variance parameter of GP) as $\lambda^2 = 1.0$ and $\lambda^2 = 10^{-8}$ in CGP-UCB and CGP-UCB-Ridgeless, respectively. Furthermore, we set the confidence width parameter $\beta_t^{(i)}$ of $f_{\ast}^{(i)}$ at step $t$ as $\beta_t^{(i)} = \|f_{\ast}^{(i)}\|_{\mH} + \sigma^2 \lambda^{-1}\sqrt{2 \ln \det(\bI + \lambda^{-2} K(\bX_t^{(i)}, \bX_{t}^{(i)})) + 2 \ln (K/\delta)}$, 
    where $\bX_{t}^{(i)}$ is the set of chosen input vector to observe $f_{\ast}^{(i)}$ until step $t$.
    This value of $\beta_t^{(i)}$ is theoretically suggested in the subsequent works on CGP-UCB~\citep[e.g.,][]{abbasi2013online,chowdhury2017kernelized}.
    \item \textbf{CGP-UCB-Scaled/CGP-UCB-Scaled-Ridgeless} are variants of the previously introduced CGP-UCB algorithm. In these versions, the value of $\beta_t^{(i)}$ is $10$ times smaller.
    \item \textbf{EtC-Linear/EtC-Linear-Ridge} are variants of the EtC algorithm that adopts a linear kernel~\citep{DBLP:conf/nips/KomiyamaI23}. We also confirm the performance of EtC-Linear-Ridge, whose $\hat{f}^{(i)}$ is defined as a linear ridge estimator with regularization parameter $\lambda^2 = 1.0$.
    \item \textbf{EtC/EtC-Ridge} are our proposed EtC algorithm. To confirm the effect of our adopted ridge-less estimator $\hat{f}^{(i)}$, we also confirm the performance of EtC-Ridge, whose $\hat{f}^{(i)}$ is defined as the kernel-ridge estimator with regularization parameter $\lambda^2 = 1.0$.
\end{itemize}

\subsection{Reward Setup}

We use the reward function 
$f_{\ast}^{(i)}(\cdot) = \sum_{m=1}^{500} c_m^{(i)} F_{\mathrm{RBF}}(\cdot, x_m^{(i)})$, where $c_m^{(i)} \sim \mathrm{Uniform}([-1, 1])$, 
and $x_m^{(i)} \sim \mN(0, I_d)$ are randomly generated coefficients
and base inputs, respectively. Here, we adopt the RBF function: $F_{\mathrm{RBF}}(x, \tilde{x}) = \exp\rbr{-\frac{\|x - \tilde{x}\|_2^2}{0.25d}}$. Then, note that the RKHS norm 
$\|f_{\ast}^{(i)}\|_{\mH}$ can be computed exactly as 
$\|f_{\ast}^{(i)}\|_{\mH} = \sqrt{\sum_{j,l} c_j^{(i)} c_l^{(i)} F_{\mathrm{RBF}}(x_j^{(i)}, x_l^{(i)})}$, and we use this exact RKHS norm for the calculation of the confidence width parameter of CGP-UCB.
We set $K = 20$ and $\xi(t) \sim \mN(0, \sigma^2)$ with $\sigma^2 = 10^{-4}$. 
Furthermore, we define the context of the $i$-th arm $X_1^{(i)}$ 
as the clipped multivariate normal random vector: $X_j^{(i)} \coloneqq \min\{10, \max\{\tilde{X}_j^{(i)}, -10\}\}$, where $\tilde{X}_j^{(i)} \sim \mN(0, \Sigma_j^{(i)})$.
We confirm the performance in the following three settings of $\Sigma_j^{(i)}$:

\begin{itemize}
  \setlength{\parskip}{0cm} 
  \setlength{\itemsep}{0cm} 
    \item \textbf{Low-rank}: The covariance matrix $\Sigma_d^{(i)}$ is defined as $\Sigma_d^{(i)} = \mathrm{diag}(\tilde{c}_1^{(i)}, \ldots, \tilde{c}_d^{(i)})$, where $\tilde{c}_j^{(i)} \coloneqq \tilde{c}^{(i)} \sim \mathrm{Uniform}([0.5, 1.0])$ for randomly chosen dimension $j \in [d]$; otherwise, $\tilde{c}_j^{(i)} = 0$. We set the number of total active dimensions as $d/2$.
    \item \textbf{Approximate low-rank}: The covariance matrix $\Sigma_d^{(i)}$ is defined as $\Sigma_d^{(i)} = \tilde{c}^{(i)} \mathrm{diag}(1, \frac{1}{2}, \frac{1}{2}, \ldots, \frac{1}{2})$, where $\tilde{c}^{(i)} \sim \mathrm{Uniform}([0.5, 1.0])$.
    \item \textbf{Spectral decay}: The covariance matrix $\Sigma_d^{(i)}$ is defined as $\Sigma_d^{(i)} = \tilde{c}^{(i)} \mathrm{diag}(\tilde{c}_1^{(i)}, \ldots, \tilde{c}_d^{(i)})$, where $\tilde{c}^{(i)} \sim \mathrm{Uniform}([0.5, 1.0])$. Here, we set $\tilde{c}_j^{(i)} = 10 j^{-1}$ for $j \in [\tilde{j}]$, where $\tilde{j} = \max\{j \in [d] \mid \sum_{l=1}^j 10 l^{-1} \leq d/4\}$; otherwise, 
    $\tilde{c}_d^{(i)} = (0.5d - \sum_{l=1}^{\tilde{j}} 10 l^{-1})/(d - \tilde{j})$.
    Note that, under this setup of $\Sigma_d^{(i)}$, the average eigenvalue $\mathrm{Tr}(\Sigma_d^{(i)})/d$ is $\tilde{c}^{(i)}/2$, which is similar to the aforementioned two covariance settings.
\end{itemize}

\subsection{Details of experiment in Avazu CTR dataset}
\label{sec:avazu_detail}

We evaluate our algorithm using the Avazu click-through rate (CTR) prediction dataset~\citep{zhu2021open}. Based on this data, we create a simulated setup for the online ad-selection problem for maximizing clicks per mile (CPM). 
First, we group advertisements by the last digit of their ad identifier to define 10 distinct ad groups, which we treat as the arms of the bandit problem. To define the context features, we apply one-hot encoding to all features in the original dataset and then perform principal component analysis (PCA) using $20000$ samples to obtain a 100-dimensional embedding feature. For each ad group, we train a kernel ridge regression model using 2,000 samples with a regularization parameter $0.1$. Then, we define these pre-trained models as a surrogate oracle for the true CTR function of the simulated experiment. During the experiment runs, rewards are generated by querying these oracle models, whereas the context vectors are taken from held-out test data, embedded into the pre-trained 100-dimensional space. We report average regret over $5$ different seeds with $T=3000$ and $N=1500$. 

\subsection{Additional experiment}
To observe the algorithm's robustness, we conduct an ablation study by changing the number of active dimensions in the low-rank reward setup. Figure~\ref{fig:low_rank_ablation} shows the results. To simplify the experiment, we focus on the ridgeless-regressor-based algorithms. The results demonstrate that, compared with CGP-UCB, our EtC algorithms are robust against the varying number of active dimensions. Specifically, while CGP-UCB exhibits strong performance in the simplest one-dimensional setting, its performance deteriorates significantly as the number of active dimensions increases. In high-dimensional settings ($10$ or more active dimensions), we can confirm that the regret of CGP-UCB becomes nearly linear.
\begin{figure}
    \centering
    \includegraphics[width=0.4\linewidth]{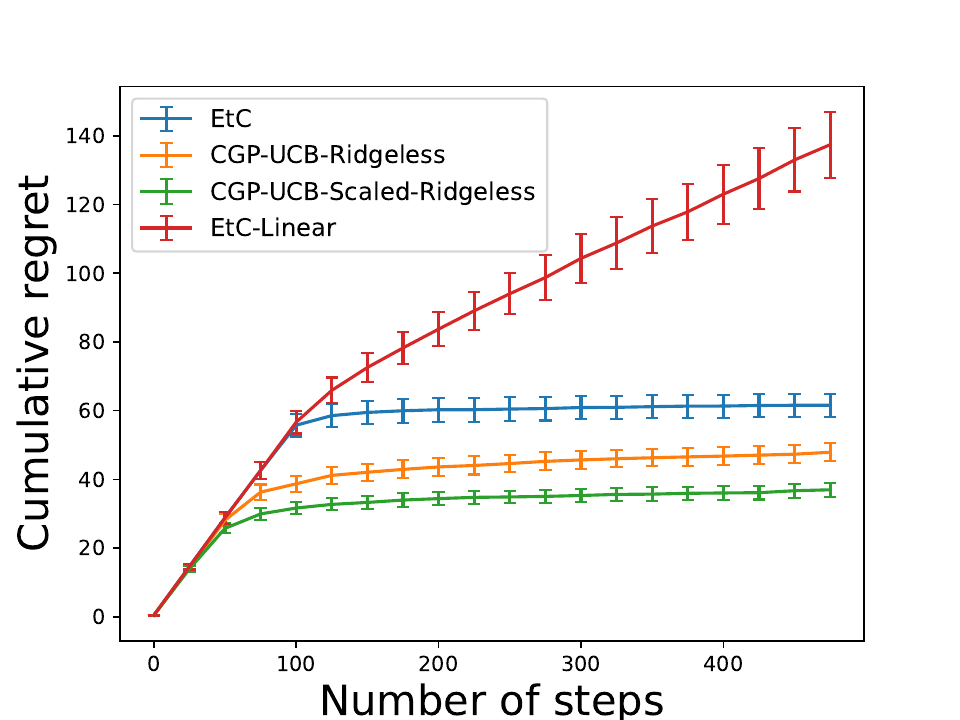}
    \includegraphics[width=0.4\linewidth]{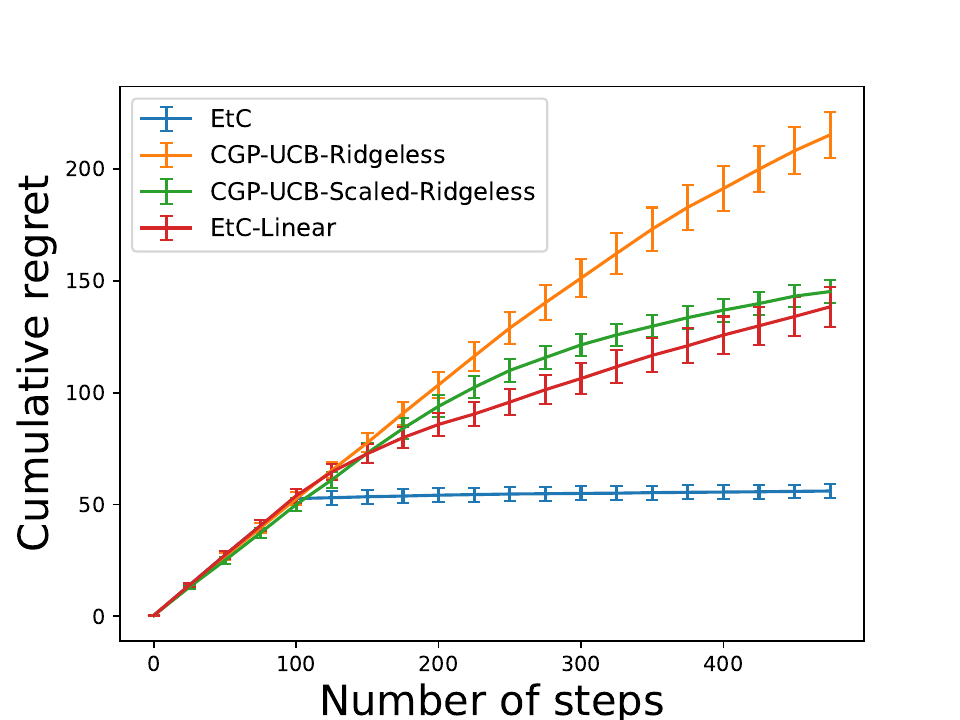}
    \includegraphics[width=0.4\linewidth]{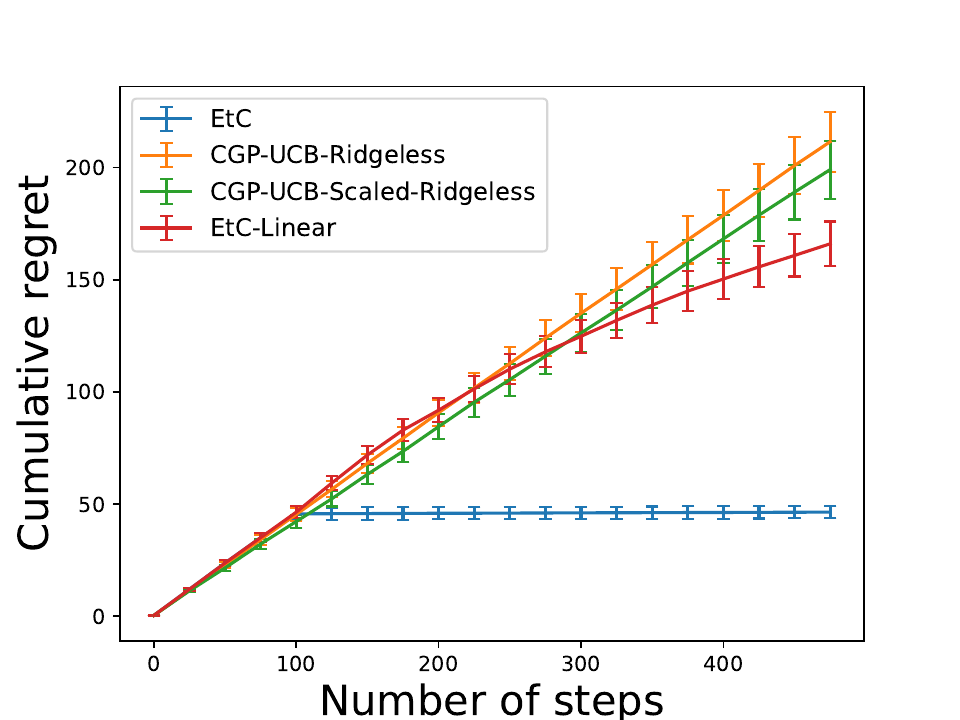}
    \includegraphics[width=0.4\linewidth]{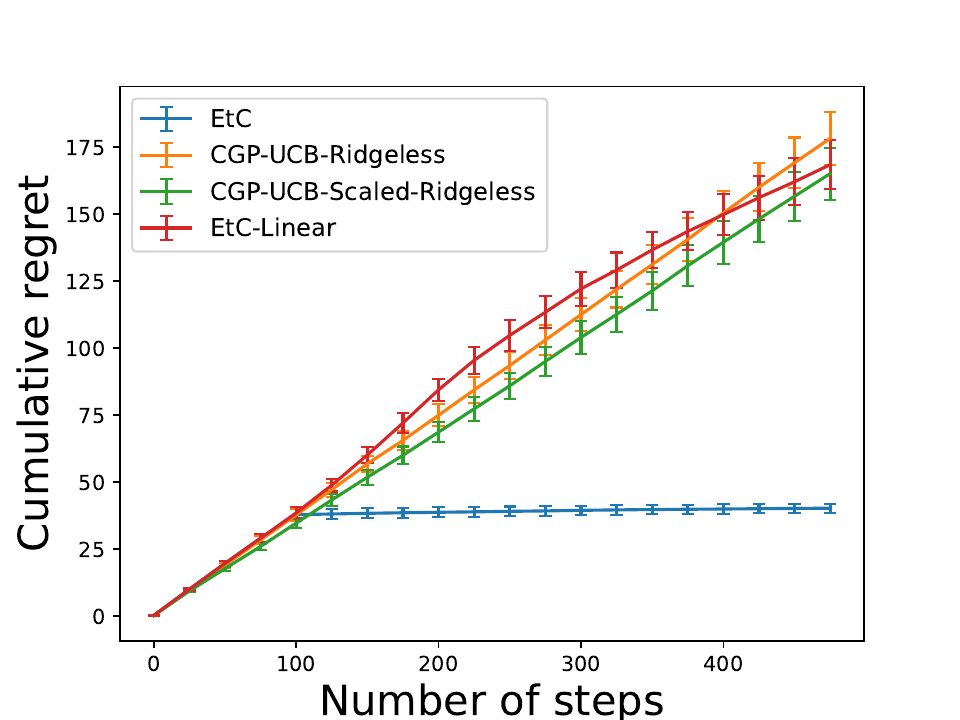}
    \caption{Experiment in the low-rank reward setup with the varying number of active dimensions. These plots report the average cumulative regret over $10$ random seeds. The top-left, top-right, bottom-left, and bottom-right plots correspond to settings with $1, 3, 10$, and $20$ active dimensions, respectively.
    }
    \label{fig:low_rank_ablation}
\end{figure}

\section{Details of Theorem \ref{thm:inner-product}} \label{sec:detail_theorem}

\subsection{Detailed statements}
We provide the detailed version of Theorem \ref{thm:inner-product} proofs for the case of the inner-product class and the case of the RBF class.

For the inner-product class case, our result simply follows the result of \cite{liang2020just}.
We give the statement as follows:
\begin{theorem}[Inner-product case of Theorem \ref{thm:inner-product}] \label{thm:liang_rakhlin}
    Fix any $i \in [K]$.
    Suppose $c_L \leq d/N \leq c_U$ holds with some universal constants $c_L, c_U \in (0, \infty)$. 
    Let $(\delta_d)_{d \in \N_+}$ be the sequence such that $\delta_d \geq 16 d^{-1} (\gamma_d^{(i)})^{-2}$ and $\delta_d = o(1)$. Then, under Assumption~\ref{asmp:basic} with the inner-product class of kernels, with probability at least $1 - 2\delta_d - d^{-2}$, the following inequalities hold for sufficiently large $d$:
    \begin{equation}
        \label{eq:gen_ub_kapprox}
        \Ep [\|\hat{f}^{(i)} - f_\ast^{(i)}\|_{L_i^2}^2 \mid \bX^{(i)}] 
         \leq  \bar{V}_{d,N}^{(i)} + \bar{B}_{d,N}^{(i)} + \bar{\Phi}_{d,N}^{(i)} ~~\mathrm{and}~~\|K(\bX^{(i)}, \bX^{(i)}) - K^{\mathrm{lin}}(\bX^{(i)}, \bX^{(i)})\|_{\mathrm{op}} \leq \frac{\gamma_d^{(i)}}{2},
    \end{equation}
    where the effective variance $\bar{V}_{d,N}^{(i)}$, bias$\bar{B}_{d,N}^{(i)}$, and the residual $\bar{\Phi}_{d,N}^{(i)}$ are defined as follows:
    \begin{align}
        \bar{V}_{d,N}^{(i)} &= \frac{8 \sigma^2}{d} \sum_{j=1}^N \frac{\lambda_j\rbr{\hat{\Sigma}_d^{(i)} + \frac{\alpha_d^{(i)}}{\beta}1 1^\top} }{\sbr{\frac{\gamma_d^{(i)}}{\beta} + \lambda_j\rbr{\hat{\Sigma}_d^{(i)} + \frac{\alpha_d^{(i)}}{\beta} 1 1^\top } }^2}, \\
        \bar{B}_{d,N}^{(i)} &= B^2 \inf_{0 \leq k \leq N}\sbr{\frac{1}{N} \sum_{j > k} \lambda_j (K_X^{(i)}) + 2\sqrt{\frac{k}{N}}}, \\
        \bar{\Phi}_{d,N}^{(i)} &= O\rbr{\frac{\ln^{4.1} d}{d (\gamma_d^{(i)})^2} + \sqrt{\frac{\ln({N/\delta_d})}{N}}}.
    \end{align} 
    Furthermore, the linear approximation $K^{\mathrm{lin}}$ of gram matrix is defined as
    \begin{equation}
        K^{\mathrm{lin}}(\bX^{(i)}, \bX^{(i)}) = \gamma_d^{(i)} I + \alpha_d^{(i)} 11^{\top} + \beta\hat{\Sigma}_d^{(i)}.
    \end{equation}
\end{theorem}

The only notable difference between the original Theorem~1 in \citet{liang2020just} and ours is that we allow the confidence level $\delta_d$ to vary depending on $d$. We require the condition $\delta_d = \Omega(d^{-1} (\gamma_d^{(i)})^{-2})$ so as to $d^{-1/2} (\delta_d^{-1/2} + \ln^{0.51} d) \leq \gamma_d^{(i)}/2$ holds for sufficiently large $d$ to follow the lines in the proof of Theorem 2 in \citet{liang2020just}.

We next discuss the case with the RBF class.
We first define a detailed version of the parameters in Definition \ref{def:kernel_parameters}.
The details are only used for the case with the RBF class.
\begin{definition}[Parameter for RBF class] \label{def:parameters}
    Define $\tau_d^{(i)} := 2\mathrm{tr}(\Sigma_d^{(i)}) / d$, $\psi \in \R^n$ with $\psi_j := \|X_j^{(i)}\|^2/d - \tau_d^{(i)} / 2$, and $\eta_d \searrow 0$ as $d \to \infty$ such that $\eta_d \geq \|\psi\|$ holds.
    For any $i \in [K]$, define $\alpha_d^{(i)}$, $\beta_d^{(i)}$, $\gamma_d^{(i)}$, and $\zeta_d^{(i)}$ as follows:
    \begin{align}
        \alpha_d^{(i)} &= h(\tau_d^{(i)})  + 2h''(\tau_d^{(i)})\mathrm{Tr}\rbr{(\Sigma_d^{(i)})^2)}/d^2,  ~
        \beta_d^{(i)} =-2 h'(\tau_d^{(i)}), \\
        r_d^{(i)} &= \frac{\gamma_d^{(i)}}{\beta_d^{(i)}}, ~\gamma_d^{(i)} = h(0) + \tau_d^{(i)} h'(\tau_d^{(i)}) - h(\tau_d^{(i)}), \\
        \rho_d^{(i)} &= h'(\tau_d^{(i)}) \psi + h''(\tau_d^{(i)}) \psi \circ \psi/ 2, ~ \zeta_d^{(i)} = h'(\tau_d^{(i)}) \psi (\psi)^\top.
    \end{align}
    $\circ$ denotes a element-wise product.
\end{definition}

\begin{remark}
\label{rem:coefficients}
Suppose $\Sigma_d^{(i)}$ satisfies $\trace(\Sigma_d^{(i)}) = O(d^{1/2 + a})$ with some $a \in (0,1/2)$, $\|X_i\| = O_P(d^b)$ with some $b \in \mathbb{R}$, and $h(\cdot)$ is Lipschitz continuous.
Then, we have $\tau_d^{(i)} = O(d^{-1/2 + a})$, $\alpha_d^{(i)} = h(0) + O(d^{-1/2 + a})$, $\gamma_d^{(i)} = O(d^{-1/2 + a})$, $\|\rho_d^{(i)}\| = O(d^{b})$, and $\|\zeta_d^{(i)}\|_\mathrm{op} = O(d^{2b})$.    
\end{remark} 

The following \thmref{lem:gen_err} describes the generalization error bounds 
adapted from Theorem~1 in \citet{liang2020just}. 
This proposition comes from the following lemma.

\begin{theorem}[RBF class case of Theorem \ref{thm:inner-product}]
    \label{lem:gen_err}
    Fix any $i \in [K]$.
    Suppose $c_L \leq d/N \leq c_U$ holds with some universal constants $c_L, c_U \in (0, \infty)$. 
    Furthermore, define $\delta_d \in (0, 1)$ such that
    $\delta_d = \omega(d^{-1} (\gamma_d^{(i)})^{-2})$ and $\delta_d = o(1)$ hold.
    Then, under Assumption~\ref{asmp:basic} with the RBF class of kernels, with probability at least $1 - 2\delta_d - O(d^{-2}) -\eta_d$, the following inequality holds for sufficiently large $d$:
    \begin{equation}
        \label{eq:gen_ub_kapprox_rbf}
        \Ep [\|\hat{f}^{(i)} - f_\ast^{(i)}\|_{L_i^2}^2 \mid \bX^{(i)}] 
         \leq  \bar{V}_{d,N}^{(i)} + \bar{B}_{d,N}^{(i)} + \Phi_{d,N}^{(i)}, ~~\mathrm{and}~~\|K(\bX^{(i)}, \bX^{(i)}) - K^{\mathrm{lin}}(\bX^{(i)}, \bX^{(i)})\|_{\mathrm{op}} \leq \frac{\gamma_d^{(i)}}{2}, 
    \end{equation}
    where the variance term $\bar{V}_{d,N}^{(i)}$ and the bias term $\bar{B}_{d,N}^{(i)}$, and the residual ${\Phi}_{d,N}^{(i)}$ are defined as follows:
    \begin{align}
        \bar{V}_{d,N}^{(i)} &= \frac{8 \sigma^2}{d} \sum_{j=1}^N \frac{\lambda_j\rbr{\hat{\Sigma}_d^{(i)} + \frac{\alpha_d^{(i)}}{\beta_d^{(i)}} 1 1^\top} }{\sbr{r_d^{(i)} + \lambda_j\rbr{\hat{\Sigma}_d^{(i)} + \frac{\alpha_d^{(i)}}{\beta_d^{(i)}}1 1^\top + \frac{1}{\beta_d^{(i)}}\rho_d^{(i)} 1^\top  + \frac{1}{\beta_d^{(i)}} 1 (\rho_d^{(i)})^\top + \frac{1}{\beta_d^{(i)}} \zeta_d^{(i)}} }^2}, \\
        \bar{B}_{d,N}^{(i)} &= B^2 \inf_{0 \leq k \leq N}\sbr{\frac{1}{N} \sum_{j > k} \lambda_j (K_X^{(i)}) + 2 \sqrt{\frac{k}{N}}}, \\
        \Phi_{d,N}^{(i)} &= O\rbr{\frac{\ln^{4.1} d}{d (\gamma_d^{(i)})^2} + \sqrt{\frac{\ln({N/\delta_d})}{N}}}.
    \end{align}
    Furthermore, the linear approximation $K^{\mathrm{lin}}$ of gram matrix is defined as
    \begin{equation}
        K^{\mathrm{lin}}(\bfXi,\bfXi) := \gamma_d^{(i)} I + \alpha_d^{(i)} 1_d 1_d^\top + \beta_d^{(i)} \bfXi (\bfXi)^\top / d + \rho_d^{(i)} 1_d^\top  + 1_d (\rho_d^{(i)})^\top + \zeta_d^{(i)}.\label{eq:kernel_lin_rbf}
    \end{equation}
\end{theorem}

In the case of Remark \ref{rem:coefficients}, we obtain an explicit rate as follows:
Namely, the coefficient $\delta_d$ in Theorem \ref{lem:gen_err} has an order $\delta_d = O(d^{-2a})$.
Also, if we set $\eta_d = O(d^{b})$, the statement of Theorem \ref{lem:gen_err} holds with probability at least $1 - 2\delta_d - d^{-2} -\eta_d$.

\subsection{Detailed examples}
\label{subsec_detail_expl}

We discuss an explicit convergence rate of Theorem \ref{thm:liang_rakhlin} and Theorem \ref{lem:gen_err}.

\begin{lemma}[Detailed version of Case I in \exref{ex:liang_sec4} with inner-product class of kernels, adapted from Example 4.1 in \citet{liang2020just}]
\label{lem:low_rank} 
    Suppose $N > d$, and $\Sigma_d^{(i)} = \mathrm{diag}(1,\dots,1,0,\dots,0)$ to be a sparse diagonal 
    matrix with $\lfloor \varepsilon d \rfloor$ ones $1$ where $\varepsilon \in [1/d, 1)$.
    Then, under Assumption~\ref{asmp:basic} with the inner-product class of kernels, sufficiently small $\varepsilon d/N$, and event ~\eqref{eq:gen_ub_kapprox},
    \begin{equation}
        \bar{V}_{d,N}^{(i)} \leq 32\sigma^2 \varepsilon \frac{d}{N} + O\rbr{\frac{1}{d\varepsilon^2}} ~~\mathrm{and}~~\bar{B}_{d,N}^{(i)} \leq  
        B^2 \rbr{\frac{3 h_{\mathrm{max}}'' \varepsilon^2}{4} + \beta \varepsilon} + O\rbr{\sqrt{\frac{\ln (1/\bar{\delta})}{N}}}
    \end{equation}
    hold with probability at least $1 - 2\bar{\delta}$ for any $\bar{\delta} \in (0, 1)$. Here, we define
    $h_{\mathrm{max}}''$ as
    $h_{\mathrm{max}}'' = \mathrm{max}_{a \in [0, 1]} h''(a)$.
\end{lemma}

\begin{lemma}[Detailed version of Case II in \exref{ex:liang_sec4} with inner-product class of kernels, adapted from Example 4.2 in \citet{liang2020just}]
\label{lem:approx_low_rank}
    Suppose $N > d$ and $\Sigma_d^{(i)} = \text{diag}(1, \varepsilon, \ldots, \varepsilon)$ for $\varepsilon \in (0, 1)$.
    Then, under Assumption~\ref{asmp:basic} with the inner-product class of kernels, sufficiently small $d/N$, and event ~\eqref{eq:gen_ub_kapprox},
    \begin{align}
        \bar{V}_{d,N}^{(i)} \leq 32\sigma^2 \varepsilon^{-1} \frac{d}{N} + O\rbr{\frac{1}{d\varepsilon^2}}
        ~~\mathrm{and}~~\bar{B}_{d,N}^{(i)} \leq B^2 \rbr{ \frac{3 h_{\mathrm{max}}''}{2} \varepsilon^2 + \beta\varepsilon} + O\rbr{\frac{1}{d} + \sqrt{\frac{\ln (1/\bar{\delta})}{N}}} 
    \end{align}
    hold with probability at least $1 - 2\bar{\delta}$ for any $\bar{\delta} \in (0, 1)$.
\end{lemma}

\begin{lemma}[Detailed version of Case III in \exref{ex:liang_sec4} with inner-product class of kernels, adapted from Example 4.4 in \citet{liang2020just}]
\label{lem:N_ll_d}
    Suppose $d > N$ and $\mathrm{Tr}(\Sigma_d^{(i)})/d = \varepsilon$ for $\varepsilon \in (0, 1)$.
    Then, under Assumption~\ref{asmp:basic} with the inner-product kernel and event ~\eqref{eq:gen_ub_kapprox},
    \begin{align}
        \bar{V}_{d,N}^{(i)}  \leq \frac{8\sigma^2 \beta}{ \varepsilon^2 h_{\mathrm{min}}''} \rbr{\frac{d}{N}}^{-1} + O\rbr{\frac{1}{d\varepsilon^2}}
        ~~\mathrm{and}~~
        \bar{B}_{d,N}^{(i)} \leq B^2 \rbr{ \frac{3 h_{\mathrm{max}}''}{4} \varepsilon^2 + \beta \varepsilon} + O\rbr{\sqrt{\frac{\ln (1/\bar{\delta})}{N}}}
    \end{align}
    hold with probability at least $1 - 2\bar{\delta}$ for any $\bar{\delta} \in (0, 1)$.
    Here, $h_{\mathrm{min}}'' = \min_{a \in [0, 1]} h''(a) > 0$.
\end{lemma}

To prove the above lemmas, we first show the following lemma. We also note that the statement of \thmref{thm:inner-product} for the inner-product class directly follows from \lemref{lem:exact_ub_VB} and \thmref{thm:liang_rakhlin} with the proper setting of $\delta_d$.
\begin{lemma}
    \label{lem:exact_ub_VB}
    Under event ~\eqref{eq:gen_ub_kapprox}, the following inequality holds for sufficiently large $d$:
    \begin{align}
        \bar{V}_{d,N}^{(i)} &\leq \frac{8\sigma^2}{d} \sbr{\frac{\beta}{\gamma_d^{(i)}} + 2 \sum_{j=1}^N \frac{\lambda_j\rbr{\hat{\Sigma}_d^{(i)}} }{\sbr{\frac{\gamma_d^{(i)}}{\beta} + \lambda_j\rbr{\hat{\Sigma}_d^{(i)}} }^2} }, \\
        \bar{B}_{d,N}^{(i)} &\leq B^2 \sbr{\frac{3\gamma_d^{(i)}}{2} + \frac{\beta}{N} \sum_{j=1}^N \lambda_j\rbr{\hat{\Sigma}_d^{(i)}} + \sqrt{\frac{4}{N}} },
    \end{align}
    where $\bar{V}_{d,N}^{(i)}$ and $\bar{B}_{d,N}^{(i)}$ are defined in \thmref{thm:liang_rakhlin}.
\end{lemma}
\begin{proof}
    The application of Weyl's inequality implies
    \begin{align}
        &\lambda_1\rbr{\hat{\Sigma}_d^{(i)}} \leq \lambda_1\rbr{\hat{\Sigma}_d^{(i)} + \frac{\alpha_d^{(i)}}{\beta} 11^{\top}} \leq \lambda_1\rbr{\hat{\Sigma}_d^{(i)}} + \frac{\alpha_d^{(i)}}{\beta} N, \\
        &\lambda_j\rbr{\hat{\Sigma}_d^{(i)}} \leq \lambda_j\rbr{\hat{\Sigma}_d^{(i)} + \frac{\alpha_d^{(i)}}{\beta} 11^{\top}} \leq \lambda_{j-1}\rbr{\hat{\Sigma}_d^{(i)}} ~~(\mathrm{for~all~}j \in [N]\setminus \{1\}).
    \end{align}
    From the above inequalities and the maximum of the function $\frac{x}{\gamma_d^{(i)}/\beta + x}$, we can easily confirm
    \begin{equation}
        \sum_{j=1}^N \frac{\lambda_j\rbr{\hat{\Sigma}_d^{(i)} + \frac{\alpha_d^{(i)}}{\beta} 1 1^\top}}{\sbr{\frac{\gamma_d^{(i)}}{\beta} + \lambda_j\rbr{\hat{\Sigma}_d^{(i)} + \frac{\alpha_d^{(i)}}{\beta} 1 1^\top}}^2} 
        \leq \frac{\beta}{\gamma_d^{(i)}} + 2 \sum_{j=1}^N \frac{\lambda_j\rbr{\hat{\Sigma}_d^{(i)}} }{\sbr{\frac{\gamma_d^{(i)}}{\beta} + \lambda_j\rbr{\hat{\Sigma}_d^{(i)}} }^2}.
    \end{equation}
    Next, regarding the effective bias $\bar{B}_{d,N}^{(i)}$, we have
    \begin{align}
        \sum_{j=2}^N \lambda_j\rbr{K(\bX^{(i)}, \bX^{(i)})}
        &\leq \sum_{j=2}^N \sbr{\lambda_j\rbr{K^{\mathrm{lin}}(\bX^{(i)}, \bX^{(i)})} + \frac{\gamma_d^{(i)}}{2}} \\
        &\leq \sum_{j=2}^N \sbr{\beta \lambda_j\rbr{\hat{\Sigma}_d^{(i)}} + \frac{3 \gamma_d^{(i)}}{2}} \\
        &\leq \frac{3N\gamma_d^{(i)}}{2} + \beta \sum_{j=1}^N \lambda_j\rbr{\hat{\Sigma}_d^{(i)}},
    \end{align}
    where the first inequality follows from event \eqref{eq:gen_ub_kapprox} and Weyl's inequality, and the second inequality follows from the definition of $K^{\mathrm{lin}}(\bX^{(i)}, \bX^{(i)})$. The above inequalities imply the desired result by setting $k = 1$ in \thmref{thm:inner-product}.
\end{proof}

\begin{proof}[Proof of \lemref{lem:low_rank}]
    Following the same arguments as Example 4.1 and (B.6) in \citet{liang2020just}, we have
    \begin{equation}
        \frac{1}{d} \sum_{j=1}^N \frac{\lambda_j\rbr{\hat{\Sigma}_d^{(i)}} }{\sbr{\frac{\gamma_d^{(i)}}{\beta} + \lambda_j\rbr{\hat{\Sigma}_d^{(i)}} }^2} \leq 2\varepsilon \frac{d}{N}
    \end{equation}
    with probability at least $1 - \bar{\delta}$ for sufficiently small $\varepsilon d/N$. 
    Furthermore, as with the arguments in the proof of  Corollary 4.1 in \citet{liang2020just}\footnote{The original proof of \citet{liang2020just} is based on $\chi^2$ concentration. Our proof considers replacing this part with the Bernstein's inequality with the boundness assumption of the context (\asmpref{asmp:basic}).}, 
    \begin{equation}
        \label{eq:eigen_sum}
        \frac{\beta}{N} \sum_{j=1}^N \lambda_j\rbr{\hat{\Sigma}_d^{(i)}} 
        \leq \beta \frac{\mathrm{Tr}(\Sigma_d^{(i)})}{d} + O\rbr{\sqrt{\frac{\ln (1/\bar{\delta})}{N}}}
    \end{equation}
    with probability at least $1 - \bar{\delta}$. 
    Here, from Taylor's theorem and the definition of $\gamma_d^{(i)}$, we have 
    \begin{equation}
        \label{eq:gamma_bound}
        \frac{1}{2} \rbr{\frac{\mathrm{Tr}(\Sigma_d^{(i)})}{d}}^2 h_{\mathrm{min}}'' \leq \gamma_d^{(i)} \leq \frac{1}{2} \rbr{\frac{\mathrm{Tr}(\Sigma_d^{(i)})}{d}}^2 h_{\mathrm{max}}''.
    \end{equation}
    Furthermore, from the definition of $\Sigma_d^{(i)}$, $\varepsilon/2 \leq \mathrm{Tr}(\Sigma_d^{(i)})/d \leq \varepsilon$.
    By aligning the aforementioned inequalities with \lemref{lem:exact_ub_VB}, we have the desired result for $\bar{V}_{d,N}^{(i)}$ as follows:
    \begin{align}
        \bar{V}_{d,N}^{(i)} &\leq \frac{8\sigma^2}{d} \sbr{\frac{\beta}{\gamma_d^{(i)}} + 2 \sum_{j=1}^N \frac{\lambda_j\rbr{\hat{\Sigma}_d^{(i)}} }{\sbr{\frac{\gamma_d^{(i)}}{\beta} + \lambda_j\rbr{\hat{\Sigma}_d^{(i)}} }^2} } \\
        &\leq O\rbr{\frac{1}{d \varepsilon^2}} + \frac{16 \sigma^2}{d} \sum_{j=1}^N \frac{\lambda_j\rbr{\hat{\Sigma}_d^{(i)}} }{\sbr{\frac{\gamma_d^{(i)}}{\beta} + \lambda_j\rbr{\hat{\Sigma}_d^{(i)}} }^2} \\
        &\leq O\rbr{\frac{1}{d \varepsilon^2}} + 32 \sigma^2 \varepsilon \frac{d}{N}.
    \end{align}
    Furthermore,
    \begin{align}
        \bar{B}_{d,N}^{(i)} &\leq B^2 \sbr{\frac{3\gamma_d^{(i)}}{2} + \frac{\beta}{N} \sum_{j=1}^N \lambda_j\rbr{\hat{\Sigma}_d^{(i)}} + \sqrt{\frac{4}{N}} } \\
        &\leq B^2 \rbr{ \frac{3h_{\mathrm{max}}''}{4} \varepsilon^2 + \beta\varepsilon} + O\rbr{\sqrt{\frac{\ln (1/\bar{\delta})}{N}}} + B^2 \sqrt{\frac{4}{N}} \\
        &\leq B^2 \rbr{ \frac{3h_{\mathrm{max}}''}{4} \varepsilon^2 + \beta\varepsilon} + O\rbr{\sqrt{\frac{\ln (1/\bar{\delta})}{N}}}.
    \end{align}
    
\end{proof}

\begin{proof}[Proof of \lemref{lem:approx_low_rank}]
    Following the same arguments as Example 4.2 in \citet{liang2020just}, we have
    \begin{equation}
        \sum_{j=1}^N \frac{\lambda_j\rbr{\hat{\Sigma}_d^{(i)}} }{\sbr{\frac{\gamma_d^{(i)}}{\beta} + \lambda_j\rbr{\hat{\Sigma}_d^{(i)}} }^2} \leq 2 \varepsilon^{-1} \frac{d}{N}
    \end{equation}
    with probability at least $1 - \bar{\delta}$ for sufficiently small $d/N$. 
    Furthermore, from the definition of $\Sigma_d^{(i)}$, we have
    \begin{equation}
        \frac{1}{4} h_{\mathrm{min}}'' \varepsilon^2 \leq \gamma_d^{(i)} \leq \frac{1}{2} h_{\mathrm{max}}'' (d^{-1} + \varepsilon)^2 \leq h_{\mathrm{max}}'' d^{-2} + h_{\mathrm{max}}'' \varepsilon^{2}.
    \end{equation}
    By combining the above inequalities with \lemref{lem:exact_ub_VB}, we obtain
    \begin{align}
        \bar{V}_{d,N}^{(i)} &\leq \frac{8\sigma^2}{d} \sbr{\frac{\beta}{\gamma_d^{(i)}} + 2 \sum_{j=1}^N \frac{\lambda_j\rbr{\hat{\Sigma}_d^{(i)}} }{\sbr{\frac{\gamma_d^{(i)}}{\beta} + \lambda_j\rbr{\hat{\Sigma}_d^{(i)}} }^2} } \\
        &\leq O\rbr{\frac{1}{d \varepsilon^2}} + \frac{16 \sigma^2}{d} \sum_{j=1}^N \frac{\lambda_j\rbr{\hat{\Sigma}_d^{(i)}} }{\sbr{\frac{\gamma_d^{(i)}}{\beta} + \lambda_j\rbr{\hat{\Sigma}_d^{(i)}} }^2} \\
        &\leq O\rbr{\frac{1}{d \varepsilon^2}} + 32 \sigma^2 \varepsilon^{-1} \frac{d}{N}.
    \end{align}
    Furthermore, from Eq.~\eqref{eq:eigen_sum}, 
    \begin{align}
        \bar{B}_{d,N}^{(i)} &\leq B^2 \sbr{\frac{3 \gamma_d^{(i)}}{2} + \frac{\beta}{N} \sum_{j=1}^N \lambda_j\rbr{\hat{\Sigma}_d^{(i)}} + \sqrt{\frac{4}{N}} } \\
        &\leq B^2 \rbr{ \frac{3 h_{\mathrm{max}}''}{2} (d^{-2} + \varepsilon^2) + \beta(d^{-1} + \varepsilon)} + O\rbr{\sqrt{\frac{\ln (1/\bar{\delta})}{N}}} + B^2 \sqrt{\frac{4}{N}} \\
        &\leq B^2 \rbr{ \frac{3 h_{\mathrm{max}}''}{2} \varepsilon^2 + \beta\varepsilon} + O\rbr{\frac{1}{d} + \sqrt{\frac{\ln (1/\bar{\delta})}{N}}}.
    \end{align}
\end{proof}

\begin{proof}[Proof of \lemref{lem:N_ll_d}]
    Since $t/(r + t)^2 \leq 1/(4r)$, we have
    \begin{align}
        \frac{\lambda_j\rbr{\hat{\Sigma}_d^{(i)}} }{\sbr{\frac{\gamma_d^{(i)}}{\beta} + \lambda_j\rbr{\hat{\Sigma}_d^{(i)}} }^2} \leq \frac{\beta}{4\gamma_d^{(i)}} \leq \frac{\beta}{2 h_{\mathrm{min}}'' \varepsilon^2}.
    \end{align}
    Hence,
    \begin{align}
        \bar{V}_{d,N}^{(i)} &\leq \frac{8\sigma^2}{d} \sbr{\frac{\beta}{\gamma_d^{(i)}} + 2 \sum_{j=1}^N \frac{\lambda_j\rbr{\hat{\Sigma}_d^{(i)}} }{\sbr{\frac{\gamma_d^{(i)}}{\beta} + \lambda_j\rbr{\hat{\Sigma}_d^{(i)}} }^2} } \\
        &\leq O\rbr{\frac{1}{d \varepsilon^2}} + \frac{8\sigma^2 \beta}{h_{\mathrm{min}}'' \varepsilon^2} \rbr{\frac{d}{N}}^{-1}.
    \end{align}
    The upper bound for $\bar{B}_{d,N}^{(i)}$ follows from the same arguments as the proof of \lemref{lem:low_rank}.
\end{proof}

Similarly to the inner-product class of kernels, the following three lemmas show the detailed versions of \exref{ex:liang_sec4}.
\begin{lemma}[Detailed version of Case I in \exref{ex:liang_sec4} with RBF class, Adapted from Example 4.1 in \citet{liang2020just}]
\label{lem:low_rank_rbf} 
    Suppose $N > d$, and $\Sigma_d^{(i)} = \mathrm{diag}(1,\dots,1,0,\dots,0)$ to be a sparse diagonal 
    matrix with $\lfloor \varepsilon d \rfloor$ ones $1$ where $\varepsilon \in [1/d, 1)$.
    Then, the same assumptions as used in \thmref{thm:etc_reg_simple_rbf}, the RBF class of kernels, sufficiently small $\varepsilon d/N$, and event ~\eqref{eq:gen_ub_kapprox_rbf},
    \begin{equation}
        \bar{V}_{d,N}^{(i)} \leq O\rbr{\frac{1}{d \varepsilon}} + 32 \sigma^2 \varepsilon \frac{d}{N}~~\mathrm{and}~~\bar{B}_{d,N}^{(i)} \leq 
         -5 B^2 \varepsilon h_{\mathrm{min}}' + O\rbr{\sqrt{\frac{\ln (1/\bar{\delta})}{N}}}
    \end{equation}
    hold with probability at least $1 - 2\bar{\delta}$ for any $\bar{\delta} \in (0, 1)$. Here, we set $h_{\mathrm{min}}' = \min_{a \in [0, 2]} h'(a) \leq 0$.
\end{lemma}

\begin{lemma}[Detailed version of Case II in \exref{ex:liang_sec4} with RBF class of kernels, Adapted from Example 4.2 in \citet{liang2020just}]
\label{lem:approx_low_rank_rbf}
    Suppose $N > d$ and $\Sigma_d^{(i)} = \text{diag}(1, \varepsilon, \ldots, \varepsilon)$ for $\varepsilon \in (0, 1)$.
    Then, under the same assumptions as used in \thmref{thm:etc_reg_simple_rbf}, the RBF class of kernels, sufficiently small $d/N$, and event ~\eqref{eq:gen_ub_kapprox_rbf},
    \begin{align}
        \bar{V}_{d,N}^{(i)} \leq O\rbr{\frac{1}{d \varepsilon}} + 32 \sigma^2 \varepsilon^{-1} \frac{d}{N}
        ~~\mathrm{and}~~
        \bar{B}_{d,N}^{(i)} \leq -5 B^2 h_{\mathrm{min}}' \varepsilon + O\rbr{\frac{1}{d} + \sqrt{\frac{\ln (1/\bar{\delta})}{N}}}
    \end{align}
    hold with probability at least $1 - 2\bar{\delta}$ for any $\bar{\delta} \in (0, 1)$.
\end{lemma}

\begin{lemma}[Detailed version of Case III in \exref{ex:liang_sec4} with RBF class of kernels, adapted from Example 4.4 in \citet{liang2020just}]
\label{lem:N_ll_d_rbf}
    Suppose $d > N$ and $\mathrm{Tr}(\Sigma_d^{(i)})/d = \varepsilon$ for $\varepsilon \in (0, 1)$.
    Then, under the same assumptions as used in \thmref{thm:etc_reg_simple_rbf}, the RBF class,  and event ~\eqref{eq:gen_ub_kapprox_rbf},
    \begin{align}
        \bar{V}_{d,N}^{(i)} \leq O\rbr{\frac{1}{d \varepsilon}} - \frac{4 h_{\mathrm{min}}' \sigma^2}{ \varepsilon \underline{c}} \rbr{\frac{d}{N}}^{-1}
        ~~\mathrm{and}~~
        \bar{B}_{d,N}^{(i)} \leq -5 B^2 \varepsilon h_{\mathrm{min}}' + O\rbr{\sqrt{\frac{\ln (1/\bar{\delta})}{N}}}
    \end{align}
    hold with probability at least $1 - 2\bar{\delta}$ for any $\bar{\delta} \in (0, 1)$.
\end{lemma}

\begin{lemma}
    \label{lem:exact_ub_VB_rbf}
    Under event ~\eqref{eq:gen_ub_kapprox}, the following inequality holds for sufficiently large $d$:
    \begin{align}
        \bar{V}_{d,N}^{(i)} &\leq \frac{8\sigma^2}{d} \sbr{\frac{\beta_d^{(i)}}{\gamma_d^{(i)}} + 2 \sum_{j=1}^N \frac{\lambda_j\rbr{\hat{\Sigma}_d^{(i)}} }{\sbr{\frac{\gamma_d^{(i)}}{\beta_d^{(i)}} + \lambda_j\rbr{\hat{\Sigma}_d^{(i)}} }^2} }, \\
        \bar{B}_{d,N}^{(i)} &\leq B^2 \sbr{\frac{3 \gamma_d^{(i)}}{2} + \frac{\beta_d^{(i)}}{N} \sum_{j=1}^N \lambda_j\rbr{\hat{\Sigma}_d^{(i)}} + \sqrt{\frac{20}{N}} },
    \end{align}
    where $\bar{V}_{d,N}^{(i)}$ and $\bar{B}_{d,N}^{(i)}$ are defined in \thmref{lem:gen_err}.
\end{lemma}
    In the above lemma, the upper bound of $V_{d,N}^{(i)}$ follows from the same arguments as the proof of \lemref{lem:exact_ub_VB}.
    The upper bound of $B_{d,N}^{(i)}$ also follows from similar arguments to the proof of \lemref{lem:exact_ub_VB}.
    The only difference is that we need to set $k = 5$ since top $5$ eigenvalue of $K^{\mathrm{lin}}(\bX^{(i)}, \bX^{(i)})^{\top}/d) + \gamma_d^{(i)}$ due to the additional term $\alpha_d^{(i)} 1 1^\top + \rho_d^{(i)} 1^\top  + 1 (\rho_d^{(i)})^\top + \zeta_d^{(i)}$ whose rank is at most $5$ (See Appendix~B in \citet{liu2021kernel} for the detailed discussion). 

    By following the proof strategy of Lemmas~\ref{lem:low_rank}--\ref{lem:N_ll_d}
    with \lemref{lem:exact_ub_VB_rbf}, we obtain Lemmas~\ref{lem:low_rank_rbf}--\ref{lem:N_ll_d_rbf}.

\begin{proof}[Proof of \lemref{lem:low_rank_rbf}]
    Following the same arguments to Example 4.1 and (B.6) in \citet{liang2020just}, we have
    \begin{equation}
        \frac{1}{d} \sum_{j=1}^N \frac{\lambda_j\rbr{\hat{\Sigma}_d^{(i)}} }{\sbr{\frac{\gamma_d^{(i)}}{\beta_d^{(i)}} + \lambda_j\rbr{\hat{\Sigma}_d^{(i)}} }^2} \leq 2\varepsilon \frac{d}{N}
    \end{equation}
    with probability at least $1 - \bar{\delta}$ for sufficiently small $\varepsilon d/N$. 
    Furthermore, as with the arguments in the proof of  Corollary 4.1 in \citet{liang2020just}\footnote{The original proof of \citet{liang2020just} is based on $\chi^2$ concentration. Our proof considers replacing this part with the Bernstein inequality with the boundness assumption of the context.},
    \begin{equation}
        \label{eq:eigen_sum_rbf}
        \frac{\beta_d^{(i)}}{N} \sum_{j=1}^N \lambda_j\rbr{\hat{\Sigma}_d^{(i)}} 
        \leq \beta_d^{(i)} \frac{\mathrm{Tr}(\Sigma_d^{(i)})}{d} + O\rbr{\sqrt{\frac{\ln (1/\bar{\delta})}{N}}}
    \end{equation}
    with probability at least $1 - \bar{\delta}$. 
    
    Here, from Taylor's theorem and the definition of $\gamma_d^{(i)}$, we have
    \begin{align}
        \label{eq:gamma_bound_rbf}
        \gamma_d^{(i)} &= h(0) + \tau_d^{(i)} h'(\tau_d^{(i)}) - h(\tau_d^{(i)}) \\
        &= \tau_d^{(i)} \sbr{h'(\tau_d^{(i)}) - h'(c\tau_d^{(i)})} \\
        &= - \tau_d^{(i)} h'(c\tau_d^{(i)}) \\
        &\leq - 2 h_{\mathrm{min}}' \frac{\mathrm{Tr}(\Sigma_d^{(i)})}{d}
    \end{align}
    for some $c \in (0, 1)$.
    Furthermore, from the definition of $\Sigma_d^{(i)}$, $\varepsilon/2 \leq \mathrm{Tr}(\Sigma_d^{(i)})/d \leq \varepsilon$.
    By aligning the aforementioned inequalities with \lemref{lem:exact_ub_VB_rbf}, we have the desired result for $\bar{V}_{d,N}^{(i)}$ as follows:
    \begin{align}
        \bar{V}_{d,N}^{(i)} &\leq \frac{8\sigma^2}{d} \sbr{\frac{\beta_d^{(i)}}{\gamma_d^{(i)}} + 2 \sum_{j=1}^N \frac{\lambda_j\rbr{\hat{\Sigma}_d^{(i)}} }{\sbr{\frac{\gamma_d^{(i)}}{\beta_d^{(i)}} + \lambda_j\rbr{\hat{\Sigma}_d^{(i)}} }^2} } \\
        &\leq O\rbr{\frac{1}{d \varepsilon}} + \frac{16 \sigma^2}{d} \sum_{j=1}^N \frac{\lambda_j\rbr{\hat{\Sigma}_d^{(i)}} }{\sbr{\frac{\gamma_d^{(i)}}{\beta_d^{(i)}} + \lambda_j\rbr{\hat{\Sigma}_d^{(i)}} }^2} \\
        &\leq O\rbr{\frac{1}{d \varepsilon}} + 32 \sigma^2 \varepsilon \frac{d}{N}.
    \end{align}
    Furthermore,
    \begin{align}
        \bar{B}_{d,N}^{(i)} &\leq B^2 \sbr{\frac{3 \gamma_d^{(i)}}{2} + \frac{\beta_d^{(i)}}{N} \sum_{j=1}^N \lambda_j\rbr{\hat{\Sigma}_d^{(i)}} + \sqrt{\frac{20}{N}} } \\
        &\leq B^2 \rbr{ -3 h_{\mathrm{min}}' \varepsilon + \beta_d^{(i)} \varepsilon} + O\rbr{\sqrt{\frac{\ln (1/\bar{\delta})}{N}}} + B^2 \sqrt{\frac{20}{N}} \\
        &\leq -5 B^2 \varepsilon h_{\mathrm{min}}' + O\rbr{\sqrt{\frac{\ln (1/\bar{\delta})}{N}}},
    \end{align}
    where the last line use $\beta_d^{(i)} = -2 h'(\tau_d^{(i)}) \leq -2 h_{\mathrm{min}}'$.
\end{proof}

\begin{proof}[Proof of \lemref{lem:approx_low_rank_rbf}]
    Following the same arguments as Example 4.2 in \citet{liang2020just}, we have
    \begin{equation}
        \sum_{j=1}^N \frac{\lambda_j\rbr{\hat{\Sigma}_d^{(i)}} }{\sbr{\frac{\gamma_d^{(i)}}{\beta_d^{(i)}} + \lambda_j\rbr{\hat{\Sigma}_d^{(i)}} }^2} \leq 2 \varepsilon^{-1} \frac{d}{N}
    \end{equation}
    with probability at least $1 - \bar{\delta}$ for sufficiently small $d/N$. 
    Furthermore, from the definition of $\Sigma_d^{(i)}$, we have
    \begin{equation}
        \gamma_d^{(i)} \leq -2 h_{\mathrm{min}}' (d^{-1} + \varepsilon)
    \end{equation}
    By combining the above inequalities with \lemref{lem:exact_ub_VB}, we obtain
    \begin{align}
        \bar{V}_{d,N}^{(i)} &\leq \frac{8\sigma^2}{d} \sbr{\frac{\beta_d^{(i)}}{\gamma_d^{(i)}} + 2 \sum_{j=1}^N \frac{\lambda_j\rbr{\hat{\Sigma}_d^{(i)}} }{\sbr{\frac{\gamma_d^{(i)}}{\beta_d^{(i)}} + \lambda_j\rbr{\hat{\Sigma}_d^{(i)}} }^2} } \\
        &\leq O\rbr{\frac{1}{d \varepsilon}} + \frac{16 \sigma^2}{d} \sum_{j=1}^N \frac{\lambda_j\rbr{\hat{\Sigma}_d^{(i)}} }{\sbr{\frac{\gamma_d^{(i)}}{\beta_d^{(i)}} + \lambda_j\rbr{\hat{\Sigma}_d^{(i)}} }^2} \\
        &\leq O\rbr{\frac{1}{d \varepsilon}} + 32 \sigma^2 \varepsilon^{-1} \frac{d}{N}.
    \end{align}
    Furthermore, from Eq.~\eqref{eq:eigen_sum_rbf}, 
    \begin{align}
        \bar{B}_{d,N}^{(i)} &\leq B^2 \sbr{\frac{3 \gamma_d^{(i)}}{2} + \frac{\beta_d^{(i)}}{N} \sum_{j=1}^N \lambda_j\rbr{\hat{\Sigma}_d^{(i)}} + \sqrt{\frac{20}{N}} } \\
        &\leq B^2 \rbr{ - 3 h_{\mathrm{min}}'(d^{-1} + \varepsilon) + \beta_d^{(i)}(d^{-1} + \varepsilon)} + O\rbr{\sqrt{\frac{\ln (1/\bar{\delta})}{N}}} + B^2 \sqrt{\frac{20}{N}} \\
        &\leq -5 B^2 h_{\mathrm{min}}' \varepsilon + O\rbr{\frac{1}{d} + \sqrt{\frac{\ln (1/\bar{\delta})}{N}}}.
    \end{align}
\end{proof}

\begin{proof}[Proof of \lemref{lem:N_ll_d_rbf}]
    Since $t/(r + t)^2 \leq 1/(4r)$, we have
    \begin{align}
        \frac{\lambda_j\rbr{\hat{\Sigma}_d^{(i)}} }{\sbr{\frac{\gamma_d^{(i)}}{\beta_d^{(i)}} + \lambda_j\rbr{\hat{\Sigma}_d^{(i)}} }^2} \leq \frac{\beta_d^{(i)}}{4\gamma_d^{(i)}} \leq 
 - \frac{h_{\mathrm{min}}'}{2 \varepsilon \underline{c}}.
    \end{align}
    Hence,
    \begin{align}
        \bar{V}_{d,N}^{(i)} &\leq \frac{8\sigma^2}{d} \sbr{\frac{\beta_d^{(i)}}{\gamma_d^{(i)}} + 2 \sum_{j=1}^N \frac{\lambda_j\rbr{\hat{\Sigma}_d^{(i)}} }{\sbr{\frac{\gamma_d^{(i)}}{\beta_d^{(i)}} + \lambda_j\rbr{\hat{\Sigma}_d^{(i)}} }^2} } \\
        &\leq O\rbr{\frac{1}{d \varepsilon}} - \frac{4 h_{\mathrm{min}}' \sigma^2}{ \varepsilon \underline{c}} \rbr{\frac{d}{N}}^{-1}.
    \end{align}
    The upper bound for $B_{d,N}^{(i)}$ follows from the same arguments as the proof of \lemref{lem:low_rank_rbf}.
\end{proof}

\section{Proofs}

\begin{lemma}[Regret upper bound of EtC with fixed $N$]
    \label{lem:reg_upper}
    Suppose that statement (i) in Assumption~\ref{asmp:basic} and $\forall x \in \Omega, K(x, x) \leq 1$ hold.
    Then, the following regret upper bound of EtC holds:
    \begin{equation}
        R(T) \leq 2BNK + 2BT\Pr(\mA^c) 
        + 2 T \sum_{i \in [K]} \Ep[\mone\{\mA\} \Ep_{\bY^{(i)}}[\|f_\ast^{(i)} - \hat{f}^{(i)}\|_{L_i^2}^2 |\bX^{(i)}]^{1/2} ],
    \end{equation}
    where $\mA$ is an arbitrary event such that $\mA \in \sigma(\bX^{(1)}, \ldots, \bX^{(K)})$. 
    Here, $\sigma(\bX^{(1)}, \ldots, \bX^{(K)})$ denotes $\sigma$-algebra generated by random vectors $\bX^{(1)}, \ldots, \bX^{(K)}$.
    Furthermore, the lenient regret upper bound $R_{\Delta}(T) \leq 2BNK$ holds if the following inequality holds: 
    \begin{equation}
        \label{eq:lr_scond}
        2B\Pr(\mA^c) + 2 \sum_{i \in [K]} \Ep[\mone\{\mA\} \Ep_{\bY^{(i)}}[\|f_\ast^{(i)} - \hat{f}^{(i)}\|_{L_i^2}^2 |\bX^{(i)}]^{1/2} ] \leq \Delta.
    \end{equation}
\end{lemma}
\begin{proof}
We first decompose the regret $R(T)$ as $R(T) = R_1 + R_2$, where:
\begin{align}
    R_1 &= \sum_{t=1}^{T_0} \Ep \left[ \left( f_*^{(i^\ast(t))}( \Xarmt{i^\ast(t)}) - f_\ast^{(\Ich)}( \Xt)\right) \right], \\
    R_2 &= \sum_{t=T_0 + 1}^{T} \Ep \left[ \left( f_*^{(i^\ast(t))}( \Xarmt{i^\ast(t)}) - f_\ast^{(\Ich)}( \Xt)\right) \right].
\end{align}
The quantities $R_1$ and $R_2$ represent the regret incurred from the exploration and exploitation phase in our EtC algorithm, respectively.
Since $\|f_\ast^{(i)}\|_{\infty} \leq \|f_\ast^{(i)}\|_{\mH} \leq B$ holds\footnote{This is a consequence of Schwarz's inequality and the reproducing property of RKHS. 
Actually, $|f_\ast^{(i)}(x)| = |\langle f_\ast^{(i)}, K(x, \cdot)\rangle_{\mH}| \leq \|f_\ast^{(i)}\|_{\mH} \|K(x, \cdot)\|_{\mH} = \|f_\ast^{(i)}\|_{\mH} K(x, x)^{1/2}$.}, 
we have $R_1 \leq 2BT_0$. 
The remaining interest is the upper bound of $R_2$.
For any $t \in [T] \setminus [T_0]$, we have
\begin{align}
    &f_\ast^{(i^\ast(t))}( \Xarmt{i^\ast(t)}) - f_\ast^{(\Ich)}( \Xt) \\
    \begin{split}
    &= f_\ast^{(i^\ast(t))}( \Xarmt{i^\ast(t)}) - \hat{f}^{(i^\ast(t))}( \Xarmt{i^\ast(t)})\\
    &\qquad+ \hat{f}^{(i^\ast(t))}( \Xarmt{i^\ast(t)}) - \hat{f}^{(\Ich)}( \Xt) \\
    &\qquad + \hat{f}^{(\Ich)}( \Xt) - f_\ast^{(\Ich)}( \Xt)
    \end{split} \\
    &\leq 2 \max_{i \in [K]} | f_\ast^{(i)}( \Xarmt{i}) - \hat{f}^{(i)}( \Xarmt{i}) |,
\end{align}
where the last inequality follows from $\hat{f}^{(i^\ast(t))}( \Xarmt{i^\ast(t)}) - \hat{f}^{(\Ich)}( \Xt) < 0$ due to the arm selection strategy of our EtC algorithm.
Here, from the tower property of the conditional expectation, we have
\begin{align}
    &\Ep_{\bY^{(i)},\Xarmt{i}}[| f_\ast^{(i)}( \Xarmt{i}) - \hat{f}^{(i)}( \Xarmt{i}) | |\bX^{(i)}] \\
    &= \Ep_{\bY^{(i)}}[\Ep_{\Xarmt{i}}[| f_\ast^{(i)}( \Xarmt{i}) - \hat{f}^{(i)}( \Xarmt{i}) | |\bX^{(i)}, \bY^{(i)}] |\bX^{(i)}] \\
    &= \Ep_{\bY^{(i)}}[\Ep_{\Xarmt{i}}[| f_\ast^{(i)}( \Xarmt{i}) - \hat{f}^{(i)}( \Xarmt{i}) |] |\bX^{(i)}] \\
    &\leq \Ep_{\bY^{(i)}}[\Ep_{\Xarmt{i}}[| f_\ast^{(i)}( \Xarmt{i}) - \hat{f}^{(i)}( \Xarmt{i}) |^2] |\bX^{(i)}]^{1/2} \\
    \label{eq:tp_cond_upper}
    &= \Ep_{\bY^{(i)}}[\|f_\ast^{(i)} - \hat{f}^{(i)}\|_{L_i^2}^2 |\bX^{(i)}]^{1/2}
    ,
\end{align}
where the third line follows from the independence between $(\bX^{(i)}, \bY^{(i)})$ and $X^{(i)}(t)$, and the fourth line follows from Jensen's inequality. By using Eq.~\eqref{eq:tp_cond_upper}, we have
\begin{align}
    &\Ep[\mone\{\mA\} | f_\ast^{(i)}( \Xarmt{i}) - \hat{f}^{(i)}( \Xarmt{i}) |] \\
    &= \Ep[[\mone\{\mA\}\Ep_{\bY^{(i)},\Xarmt{i}}[| f_\ast^{(i)}( \Xarmt{i}) - \hat{f}^{(i)}( \Xarmt{i}) | |\bX^{(i)}]] \\
    &\leq \Ep[\mone\{\mA\} \Ep_{\bY^{(i)}}[\|f_\ast^{(i)} - \hat{f}^{(i)}\|_{L_i^2}^2 |\bX^{(i)}]^{1/2} ].
\end{align}
Therefore, the following upper bound of $R_2$ holds:
\begin{align}
    R_2 
    &= \sum_{t=T_0 + 1}^T \Bigl\{ \Ep[\mone\{\mA^c\} (f_\ast^{(i^\ast(t))}( \Xarmt{i^\ast(t)}) - f_\ast^{(\Ich)}( \Xt))]  \\
    &+ \Ep[\mone\{\mA\} (f_\ast^{(i^\ast(t))}( \Xarmt{i^\ast(t)}) - f_\ast^{(\Ich)}( \Xt))] \Bigr\} \\
    &\leq 2B T \Pr(\mA^c) + 2 \sum_{t=T_0 + 1}^T \Ep[ \mone\{\mA\} \max_{i \in [K]} | f_\ast^{(i)}( \Xarmt{i}) - \hat{f}^{(i)}( \Xarmt{i}) |] \\
    &\leq 2BT \Pr(\mA^c) + 2 \sum_{t=T_0 + 1}^T \sum_{i \in [K]} \Ep[ \mone\{\mA\} | f_\ast^{(i)}( \Xarmt{i}) - \hat{f}^{(i)}( \Xarmt{i}) |] \\
    &\leq 2BT \Pr(\mA^c) + 2 \sum_{t=T_0 + 1}^T \sum_{i \in [K]} \Ep[\mone\{\mA\} \Ep_{\bY^{(i)}}[\|f_\ast^{(i)} - \hat{f}^{(i)}\|_{L_i^2}^2 |\bX^{(i)}]^{1/2} ] \\
    &\leq 2BT \Pr(\mA^c) + 2 T \sum_{i \in [K]} \Ep[\mone\{\mA\} \Ep_{\bY^{(i)}}[\|f_\ast^{(i)} - \hat{f}^{(i)}\|_{L_i^2}^2 |\bX^{(i)}]^{1/2} ].
\end{align}
Combining the upper bounds of $R_1$ and $R_2$, we obtain
the desired result. 
Finally, it is easy to confirm the lenient regret upper bound 
by noting that the following inequality holds for $t \in [T]\setminus [T_0]$:
\begin{align}
    &\Ep \left[ f_*^{(i^\ast(t))}( \Xarmt{i^\ast(t)})\right] - \Ep \left[ f_\ast^{(\Ich)}(\Xt)\right] \\
    &\leq 2B\Pr(\mA^c) + 2 \sum_{i \in [K]} \Ep[\mone\{\mA\} \Ep_{\bY^{(i)}}[\|f_\ast^{(i)} - \hat{f}^{(i)}\|_{L_i^2}^2 |\bX^{(i)}]^{1/2} ],
\end{align}
and if RHS $\le \Delta$, then the lenient regret during the exploitation period is zero.
\end{proof}

\subsection{Proof of \thmref{thm:etc_reg_simple}}

This section derives the regret bound of EtC.

\begin{proof}
    Let $V_{d,N}$, $B_{d,N}$, and $\Phi_{d,N}$ be the followings:
    \begin{align}
        V_{d,N} &= 32\sigma^2 \varepsilon \frac{d}{N} + O\rbr{\frac{1}{d\varepsilon^2}}, \\
        B_{d,N} &= B^2 \rbr{\frac{3 h_{\mathrm{max}}'' \varepsilon^2}{4} + h'(0) \varepsilon} + O\rbr{\sqrt{\frac{\ln (1/\bar{\delta})}{N}}},  \\
        \Phi_{d,N} &= O\rbr{\frac{\ln^{4.1} d}{d (\gamma_d^{(i)})^2} + \sqrt{\frac{\ln({N/\delta_d})}{N}}}.
    \end{align}
    Note that $\gamma_d^{(i)} = \Theta(\varepsilon^2) = \Theta(T^{-2\tau_2})$ from the definition of $\Sigma_d^{(i)}$. 
    Here, we set $\delta_d$ as $\delta_d = 16 d^{-1} (\gamma_d^{(i)})^{-2} = \Theta(T^{4\tau_2 - \tau_1})$. Then, from the condition $\tau_2 \in (0, \tau_1/4)$, we can confirm $\delta_d = o(1)$.
    From \thmref{thm:inner-product}, \lemref{lem:low_rank} and $N = \Theta(d)$, 
    the error bound $\Ep_{\bY^{(i)}}[\|f_\ast^{(i)} - \hat{f}^{(i)}\|_{L_i^2}^2 |\bX^{(i)}] \leq V_{d,N} + B_{d,N} + \Phi_{d,N}$ holds for any $i \in [K]$ with probability at least $1 - 4\delta_d K - d^{-2} K$. Furthermore, noting $\varepsilon = \Theta(T^{-\tau_2})$, $d = \Theta(T^{\tau_1})$, and $N = \Theta(T^{\tau_1})$, we have
    \begin{align}
        V_{d,N} &= O\rbr{\varepsilon + \frac{1}{d\varepsilon^2}} = O\rbr{T^{-\tau_2} + T^{2\tau_2 - \tau_1}}, \\
        B_{d,N} &= \tilde{O}\rbr{\varepsilon^2 + \varepsilon + \frac{1}{\sqrt{N}}} = \tilde{O}\rbr{T^{-\tau_2} + T^{-\frac{\tau_1}{2}}}, \\
        \Phi_{d,N} &= \tilde{O}\rbr{\frac{1}{d \varepsilon^4} + \frac{1}{\sqrt{N}}} = \tilde{O}\rbr{T^{4\tau_2 - \tau_1} + T^{-\frac{\tau_1}{2}}}.
    \end{align}
    Therefore, we have
    \begin{align}
        V_{d,N} + B_{d,N} + \Phi_{d,N} &= \tilde{O}\rbr{T^{-\tau_2} + T^{-\frac{\tau_1}{2}} + T^{4\tau_2-\tau_1} + T^{2\tau_2 - \tau_1}} \\
        &= \tilde{O}\rbr{T^{-\min\cbr{\frac{\tau_1}{2}, \tau_2, \tau_1 - 4\tau_2, \tau_1 - 2\tau_2}}} \\
        &= \tilde{O}\rbr{T^{-\overline{\tau}}}.
    \end{align}
    The last equation follows from $\min\cbr{\frac{\tau_1}{2}, \tau_2, \tau_1 - 4\tau_2, \tau_1 - 2\tau_2} = \min\cbr{\frac{\tau_1}{2}, \tau_2, \tau_1 - 4\tau_2} = \overline{\tau}$ since $\tau_1 / 2 < \tau_1 - 2 \tau_2$ holds for all $\tau_2 \in (0, \tau_1/4)$.
    By applying \lemref{lem:reg_upper} with the event $\mA = \{ \forall i \in [K], \Ep_{\bY^{(i)}}[\|f_\ast^{(i)} - \hat{f}^{(i)}\|_{L_i^2}^2 |\bX^{(i)}] \leq V_{d,N} + B_{d,N} + \Phi_{d,N}\}$, 
    we have
    \begin{align}
        R(T) 
        &\leq 2BNK + 2BT (4\delta_d K + d^{-2} K) + 2TK(V_{d,N} + B_{d,N} + \Phi_{d,N})^{1/2} \\
        &= \tilde{O}\rbr{T^{\tau_1} + T^{1 + 4\tau_2 - \tau_1} + T^{1 - 2\tau_1} + T^{1-\frac{1}{2} \overline{\tau}}} \\
        &= \tilde{O}\rbr{T^{\max\cbr{\tau_1, 1 - (\tau_1 - 4\tau_2), 1 - \frac{\overline{\tau}}{2}}}} \\
        &= \tilde{O}\rbr{T^{\max\cbr{\tau_1, 1 - \frac{\overline{\tau}}{2}}}},
    \end{align} 
    where the second line use $T^{1 - 2\tau_1} = O\rbr{T^{1-\frac{\tau_2}{2}}} = O\rbr{T^{1-\frac{\overline{\tau}}{2}}}$ from the definition of $\overline{\tau}$.
\end{proof}

\subsection{Proof of \thmref{thm:etc_reg_simple_rbf}}

\begin{proof}
    Let $V_{d,N}$, $B_{d,N}$, and $\Phi_{d,N}$ be the followings:
    \begin{align}
        V_{d,N} &= 32\sigma^2 \varepsilon \frac{d}{N} + O\rbr{\frac{1}{d\varepsilon}}, \\
        B_{d,N} &= -5 B^2 \varepsilon h_{\mathrm{min}}' + O\rbr{\sqrt{\frac{\ln (1/\bar{\delta})}{N}}},  \\
        \Phi_{d,N} &= O\rbr{\frac{\ln^{4.1} d}{d (\gamma_d^{(i)})^2} + \sqrt{\frac{\ln({N/\delta_d})}{N}}}.
    \end{align}
    Note that $\gamma_d^{(i)} = \Theta(\varepsilon) = \Theta(T^{-\tau_2})$ from the definition of $\Sigma_d^{(i)}$. 
    Here, we set $\delta_d$ as $\delta_d = 16 d^{-1} (\gamma_d^{(i)})^{-2} = \Theta(T^{2\tau_2 - \tau_1})$. Then, from the condition $\tau_2 \in (0, \tau_1/2)$, we can confirm $\delta_d = o(1)$.
    From \thmref{lem:gen_err}, \lemref{lem:low_rank_rbf} and $N = \Theta(d)$, 
    the error bound $\Ep_{\bY^{(i)}}[\|f_\ast^{(i)} - \hat{f}^{(i)}\|_{L_i^2}^2 |\bX^{(i)}] \leq V_{d,N} + B_{d,N} + \Phi_{d,N}$ holds for any $i \in [K]$ with probability at least $1 - 4\delta_d K - d^{-2} K - d^{-c} K$. Furthermore, noting $\varepsilon = \Theta(T^{-\tau_2})$, $d = \Theta(T^{\tau_1})$, and $N = \Theta(T^{\tau_1})$, we have
    \begin{align}
        V_{d,N} &= O\rbr{\varepsilon + \frac{1}{d\varepsilon}} = O\rbr{T^{-\tau_2} + T^{\tau_2 - \tau_1}}, \\
        B_{d,N} &= \tilde{O}\rbr{\varepsilon + \frac{1}{\sqrt{N}}} = \tilde{O}\rbr{T^{-\tau_2} + T^{-\frac{\tau_1}{2}}}, \\
        \Phi_{d,N} &= \tilde{O}\rbr{\frac{1}{d \varepsilon^2} + \frac{1}{\sqrt{N}}} = \tilde{O}\rbr{T^{2\tau_2 - \tau_1} + T^{-\frac{\tau_1}{2}}}.
    \end{align}
    Therefore, we have
    \begin{equation}
        V_{d,N} + B_{d,N} + \Phi_{d,N} = \tilde{O}\rbr{T^{-\tau_2} + T^{-\frac{\tau_1}{2}} + T^{2\tau_2-\tau_1} + T^{\tau_2 - \tau_1}} = \tilde{O}\rbr{T^{-\min\cbr{\frac{\tau_1}{2}, \tau_2, \tau_1 - 2\tau_2}}} = \tilde{O}\rbr{T^{-\overline{\tau}}}.
    \end{equation}
    By applying \lemref{lem:reg_upper} with the event $\mA = \{ \forall i \in [K], \Ep_{\bY^{(i)}}[\|f_\ast^{(i)} - \hat{f}^{(i)}\|_{L_i^2}^2 |\bX^{(i)}] \leq V_{d,N} + B_{d,N} + \Phi_{d,N}\}$, 
    we have
    \begin{align}
        R(T) 
        &\leq 2BNK + 2BT (4\delta_d K + d^{-2} K + d^{-c} K) + 2TK(V_{d,N} + B_{d,N} + \Phi_{d,N})^{1/2} \\
        &= \tilde{O}\rbr{T^{\tau_1} + T^{1 + 2\tau_2 - \tau_1} + T^{1 - 2\tau_1} + T^{1 - c\tau_1}+ T^{1-\frac{1}{2} \overline{\tau}}} \\
        &= \tilde{O}\rbr{T^{\max\cbr{\tau_1, 1 - (\tau_1 - 2\tau_2), 1-2\tau_1, 1- c\tau_1, 1 - \frac{\overline{\tau}}{2}}}} \\
        &= \tilde{O}\rbr{T^{\max\cbr{\tau_1, 1-2\tau_1, 1- c\tau_1, 1 - \frac{\overline{\tau}}{2}}}}.
    \end{align} 
\end{proof}

\subsection{Proof of \thmref{thm:etc_lr_Td}}

This section derives the lenient regret bound of EtC based on the condition for $\varepsilon$ and the choice of $T_0$ shown in Table \ref{tab:cond_eps_T0}.

\begin{table}
    \centering
    \begin{tabular}{c|ccc}
        Case & $\mathrm{I}$ & $\mathrm{II}$ & $\mathrm{III }$\\ \hline\hline
        $\varepsilon$ &$\frac{3 B^2 h_{\mathrm{max}}''}{4}\varepsilon^2 + B^2 \beta \varepsilon < \frac{\Delta}{4K}$ &  $\frac{3 B^2 h_{\mathrm{max}}''}{2}\varepsilon^2 + B^2 \beta \varepsilon < \frac{\Delta}{4K}$ &  $\frac{3 B^2 h_{\mathrm{max}}''}{2}\varepsilon^2 + B^2 \beta \varepsilon < \frac{\Delta}{4K}$\\
        $T_0$ & $T_0 = \left\lceil \frac{256 \sigma^2 K^2}{\Delta}  \varepsilon \right\rceil d$ & $T_0 = \left\lceil \frac{256 \sigma^2 K^2}{\Delta}  \varepsilon \right\rceil d$ & $T_0 = \left \lfloor \frac{\varepsilon^2 h_{\mathrm{min}}'' \Delta}{64\sigma^2 \beta}  \right\rfloor d$
    \end{tabular}
    \caption{Condition of Theorem \ref{thm:etc_lr_Td} on $\varepsilon$ and $T_0$.\label{tab:cond_eps_T0}}
\end{table}

\begin{proof}
It is enough to show that Eq.~\eqref{eq:lr_scond} holds 
with the event $\mA = \{\forall i \in [K], \Ep [\|\hat{f}^{(i)} - f_\ast^{(i)}\|_{L_i^2}^2 \mid \bX^{(i)}] \leq V_{d,N}^{(i)} + B_{d,N}^{(i)} + \Phi_{d,N}^{(i)}\}$, where $V_{d,N}^{(i)}$ and $ B_{d,N}^{(i)}$ are the upper bound of $\bar{V}_{d,N}^{(i)}$ and $\bar{B}_{d,N}^{(i)}$ described in Lemma~\ref{lem:low_rank}--\ref{lem:N_ll_d}.

Now, in each example, we can find the following from the setting of $N$ and the condition of $\varepsilon$:
\begin{itemize}
    \item In Case I (\lemref{lem:low_rank}), $V_{d,N}^{(i)} + B_{d,N}^{(i)} + \Phi_{d,N}^{(i)} = 32\varepsilon \sigma^2 \frac{d}{N} + \frac{3 B^2 h_{\mathrm{max}''}}{4}\varepsilon^2 + B^2 h'(0) \varepsilon + o(1) \leq \frac{\Delta}{8K} + \frac{\Delta}{4K} + o(1)$.
    \item In Case II (\lemref{lem:approx_low_rank}), $V_{d,N}^{(i)} + B_{d,N}^{(i)} + \Phi_{d,N}^{(i)} = 32\sigma^2 \varepsilon^{-1} \frac{d}{N} + \frac{3 B^2 h_{\mathrm{max}''}}{2}\varepsilon^2 + B^2 h'(0) \varepsilon + o(1) \leq \frac{\Delta}{8K} + \frac{\Delta}{4K} + o(1)$.
    \item In Case III (\lemref{lem:N_ll_d}), $V_{d,N}^{(i)} + B_{d,N}^{(i)} + \Phi_{d,N}^{(i)} = 8\sigma^2 \frac{h'(0) }{\varepsilon^2 h_{\mathrm{min}}''} \rbr{\frac{d}{N}}^{-1} + \frac{3 B^2 h_{\mathrm{max}''}}{4}\varepsilon^2 + B^2 h'(0) \varepsilon + o(1) \leq \frac{\Delta}{8K} + \frac{\Delta}{4K} + o(1)$.
\end{itemize}

Furthermore, by setting any $\delta_d$ such that the condition of \thmref{thm:liang_rakhlin} holds (e.g., $\delta_d = \Theta(d^{-1})$), then, 
we obtain $\Pr(\mA^c) = o(1)$. 
Therefore, for any $t \in [T]\setminus [T_0]$, we have
\begin{align}
    &2B\Pr(\mA^c) + 2 \sum_{i \in [K]} \Ep[\mone\{\mA\} \Ep_{\bY^{(i)}}[\|f_\ast^{(i)} - \hat{f}^{(i)}\|_{L_i^2}^2 |\bX^{(i)}]^{1/2} ] \\
    &\leq \frac{3}{4}\Delta + o(1).
\end{align}

Thus, for sufficiently large $T$, the condition \eqref{eq:lr_scond} holds in each example.

\end{proof}

\subsection{Lenient regret for RBF class}
\label{subsec:lr_rbf}

\begin{theorem}[Lenient regret bound for RBF class]
    \label{thm:etc_lr_Td_rbf}
    Fix any $\Delta > 0$. Suppose that Assumption~\ref{asmp:basic} holds and $\varepsilon > 0$ is constant.
    Furthermore, 
    assume that $d = \Theta(T^{\tau})$ holds for some $\tau \in (0, 1)$. Then, the lenient regret for each of the three examples in Example \ref{ex:liang_sec4} is $R_{\Delta}(T) = O(T^\tau)$ if $\varepsilon$ is small enough such that $-5B^2 h_{\mathrm{min}}' \varepsilon < \frac{\Delta}{4K}$ with $h_{\mathrm{min}}' = \min_{a \in [0, 2]} h'(a) < 0$.
    \begin{itemize}
        \item For Case I, we adopt EtC with $T_0 = \left\lceil \frac{256 \sigma^2 K^2}{\Delta}  \varepsilon \right\rceil d$.
        \item For Case II, we adopt EtC with $T_0 = \left\lceil \frac{256 \sigma^2 K^2}{\Delta \varepsilon} \right\rceil d$.
        \item For Case III, we adopt EtC with $T_0 = \left \lfloor \frac{\Delta \varepsilon \underline{c}}{32 h_{\mathrm{min}}' \sigma^2}  \right\rfloor d$.
    \end{itemize}
\end{theorem}

\begin{proof}
It is enough to show that Eq.~\eqref{eq:lr_scond} holds 
with the event $\mA = \{\forall i \in [K], \Ep [\|\hat{f}^{(i)} - f_\ast^{(i)}\|_{L_i^2}^2 \mid \bX^{(i)}] \leq V_{d,N}^{(i)} + B_{d,N}^{(i)} + \Phi_{d,N}^{(i)}\}$, where $V_{d,N}^{(i)}$ and $ B_{d,N}^{(i)}$ are the upper bound of $\bar{V}_{d,N}^{(i)}$ and $\bar{B}_{d,N}^{(i)}$ described in Lemma~\ref{lem:low_rank_rbf}--\ref{lem:N_ll_d_rbf}.

Now, in each example, we can find the following from the setting of $N$ and the condition of $\varepsilon$:
\begin{itemize}
    \item In Case I (\lemref{lem:low_rank_rbf}), $V_{d,N}^{(i)} + B_{d,N}^{(i)} + \Phi_{d,N}^{(i)} = 32\varepsilon \sigma^2 \frac{d}{N} -5 B^2 \varepsilon h_{\mathrm{min}}' + o(1) \leq \frac{\Delta}{8K} + \frac{\Delta}{4K} + o(1)$.
    \item In Case II (\lemref{lem:approx_low_rank_rbf}), $V_{d,N}^{(i)} + B_{d,N}^{(i)} + \Phi_{d,N}^{(i)} = 32\sigma^2 \varepsilon^{-1} \frac{d}{N} -5 B^2 \varepsilon h_{\mathrm{min}}' + o(1) \leq \frac{\Delta}{8K} + \frac{\Delta}{4K} + o(1)$.
    \item In Case III (\lemref{lem:N_ll_d_rbf}), $V_{d,N}^{(i)} + B_{d,N}^{(i)} + \Phi_{d,N}^{(i)} = - \frac{4 h_{\mathrm{min}}' \sigma^2}{ \varepsilon \underline{c}} \rbr{\frac{d}{N}}^{-1} -5 B^2 \varepsilon h_{\mathrm{min}}' + o(1) \leq \frac{\Delta}{8K} + \frac{\Delta}{4K} + o(1)$.
\end{itemize}

Furthermore, by setting any $\delta_d$ such that the condition of \thmref{lem:gen_err} holds (e.g., $\delta_d = \Theta(d^{-1})$), then, 
we obtain $\Pr(\mA^c) = o(1)$. 
Therefore, for any $t \in [T]\setminus [T_0]$, we have
\begin{align}
    &2B\Pr(\mA^c) + 2 \sum_{i \in [K]} \Ep[\mone\{\mA\} \Ep_{\bY^{(i)}}[\|f_\ast^{(i)} - \hat{f}^{(i)}\|_{L_i^2}^2 |\bX^{(i)}]^{1/2} ] \\
    &\leq \frac{3}{4}\Delta + o(1).
\end{align}

Thus, for sufficiently large $T$, the condition \eqref{eq:lr_scond} holds in each example.

\end{proof}

\subsection{Proofs for Theorem \ref{lem:gen_err}}

We obtain the following statement:
\begin{proposition} \label{prop:kernel_linear_approx}
    Suppose that there exists $\eta_d \searrow 0$ such that $\|\psi\|_{\infty} \leq \eta_d$ holds. Then, for sufficiently small $\delta  > 0$ and large $d \in \N$, we have the following with probability at least $1 - \delta - O(d^{-2}) - \eta_d$:
    \begin{align}
        \|K(\bfXi,\bfXi) - K^{\mathrm{lin}}(\bfXi,\bfXi)\|_{\mathrm{op}} \leq d^{-1/2} (\delta^{-1/2} +  \log^{-0.51} d).
    \end{align}
\end{proposition}
\begin{proof}
    Similarly to the proof of Proposition A.2 in \cite{liang2020just}, we refer to the proof of \cite{el2010spectrum}.
    Specifically, from the proof of Theorem 2.2 in \cite{el2010spectrum}, the approximation error is decomposed into three terms:
    \begin{align}
        K(\bfXi,\bfXi) - K^{\mathrm{lin}}(\bfXi,\bfXi) = E_1 + E_2 + E_3,
    \end{align}
    where $E_1$ is a on-diagonal first-order error, $E_2$ is an off-diagonal second-order error, and $E_3$ is an off-diagonal third order error.
    
    About the first-order term $E_1$, since it is identical to Proposition A.2 in \cite{liang2020just}, we obtain the following with sufficiently large $d$ and probability at least $1 - O(d^{-2})$:
    \begin{align}
        \|E_1\|_{\mathrm{op}} \leq 4 d^{-1/2} \log^{0.51}d.
    \end{align}

About the third-order term $E_3$, the page 27 in \cite{el2010spectrum} states that
\begin{align}
    \|E_3\|_{\mathrm{op}} = O(d^{-1/2} \log^2 d)
\end{align}
holds with probability at least $1 - O(d^{-2})$.

About the second-order term $E_2$, the proof of Proposition A.2 of \cite{liang2020just} and the proof of Theorem 2.2 of \cite{el2010spectrum} imply that
\begin{align}
    \Pr(\|E_2\|_{\mathrm{op}} \geq d^{-1/2} \delta^{-0.5} ) &\leq \Pr(\mathrm{tr}(E_2^4) \geq d^{-2} \delta^{-2} )\\
    & \leq \Ep[\mathrm{tr}(E_2^4)]d^{2} \delta^{2} \\
    & \leq C \delta^2 + \eta_d \\
    & \lesssim \delta + \eta_d.
\end{align}
The second inequality follows the Chevyshev inequality, and the third inequality follows the bound $\mathrm{tr}(E_2^4) \leq ( C d^{-2} + \eta_d$) from the proof of page 25 in \cite{el2010spectrum}.
Namely, in the RBF case, off-diagonal terms of the kernel Gram matrix appear in addition to the on-diagonal term $O(d^{-1/2})$, and \citet{el2010spectrum} shows that it has an order $O(\max_i |\psi_i|)$.
Here, we use $\max_i |\psi_i| = O (\|\psi\|_{\infty}) = O(\eta_d)$ and obtain the result.
\end{proof}

\begin{proof}[Proof of Theorem \ref{lem:gen_err}]
In this proof, we partially update the proof of Theorem 1 in \citet{liang2020just}.
Similarly to Theorem 1 in \cite{liang2020just}, we have the variance term $V_{d,N}^{(i)} $, the bias term $B_{d,N}^{(i)} $, and the residuals.
Here, the bias term of Theorem \ref{lem:gen_err} is identical to that of Theorem 1 of \cite{liang2020just}, since it does not depend on an explicit form of kernel functions.
Also, by Proposition \ref{prop:kernel_linear_approx}, the residual is guaranteed to be negligible.
Hence, we only discuss the variance term.
In preparation, using the terms in Definition \ref{def:parameters}, we remind a linearly approximated version of the kernel Gram matrices as \eqref{eq:kernel_lin_rbf} and its variant
\begin{align}
    K^{\mathrm{lin}}(\bfXi,x_j^{(i)}) := \alpha_d^{(i)} 1_d +  d^{-1}\beta_d^{(i)} \bfXi (x_j^{(i)})^\top.
\end{align}

Next, we study the difference of the matrices $\|K(\bfXi,\bfXi) - K^{\mathrm{lin}}(\bfXi,\bfXi)\|$. 
In the proof of Theorem 2 in \cite{liang2020just}, the variance term is asymptotically equivalent to $\Ep[\|K^{\mathrm{lin}}(\bfXi,\bfXi)^{-1} K^{\mathrm{lin}}(\bfXi,x_j^{(i)})\|^2]$ up to constants, by showing $\|K^{\mathrm{lin}}(\bfXi,\bfXi)^{-1} K^{\mathrm{lin}}(\bfXi,\bfXi)\|$ is bounded.
To show this equivalence by the proof, we need to derive a lower bound of the smallest eigenvalue of $K(\bfXi,\bfXi)$.
To the aim, let $\lambda_{\min}(A)$ be the smallest eigenvalue of a symmetric matrix $A$ and we have
\begin{align}
    \lambda_{\min}(K(\bfXi,\bfXi)) &= \lambda_{\min}(K^{\mathrm{lin}}(\bfXi,\bfXi) + K(\bfXi,\bfXi) - K^{\mathrm{lin}}(\bfXi,\bfXi)) \\
    &\geq \lambda_{\min}(K^{\mathrm{lin}}(\bfXi,\bfXi)) - \|K(\bfXi,\bfXi) - K^{\mathrm{lin}}(\bfXi,\bfXi)\| \\
    & \geq \lambda_{\min}(K^{\mathrm{lin}}(\bfXi,\bfXi)) - d^{-1/2} (\delta^{-1/2} +  \log^{-0.51} d) \\
    & \geq \gamma_d^{(i)} - 2 |h'(\tau_d^{(i)}) |\|\psi\| - 2 |h''(\tau_d^{(i)}) |\|\psi \circ \psi\| - d^{-1/2} (\delta^{-1/2} +  \log^{-0.51} d) \\
    & \geq \gamma_d^{(i)} - 2 |h'(\tau_d^{(i)}) |\eta_d^{(i)} - 2 |h''(\tau_d^{(i)}) |(\eta_d^{(i)})^2 - d^{-1/2} (\delta^{-1/2} +  \log^{-0.51} d) \\
    & \geq \gamma_d^{(i)} - o_P(1).
\end{align}
The second inequality follows Proposition \ref{prop:kernel_linear_approx}, and the third inequality follows the form of $K^{\mathrm{lin}}(\bfXi,\bfXi)$ in \eqref{eq:kernel_lin_rbf}.
The fourth inequality follows $\|\psi \circ \psi\| \leq \|\psi \circ \psi\|_1 = \|\psi\|^2 \leq (\eta_d^{(i)})^2$.
Hence, the smallest eigenvalues is lower bounded by $\gamma_d^{(i)}$ with sufficiently large $d$, hence we can state $\|K^{\mathrm{lin}}(\bfXi,\bfXi)^{-1} K^{\mathrm{lin}}(\bfXi,\bfXi)\|^2$ is bounded by the proof of \cite{liang2020just}.
Consequently, we can state that the variance is asymptotically bounded by  $\Ep[\|K^{\mathrm{lin}}(\bfXi,\bfXi)^{-1} K^{\mathrm{lin}}(\bfXi,x_j^{(i)})\|^2]$ up to constants.

The form is rewritten as
\begin{align}
    &\Ep[\|K^{\mathrm{lin}}(\bfXi,\bfXi)^{-1} K^{\mathrm{lin}}(\bfXi,x_j^{(i)})\|^2]\\
    &= \Ep \Bigl[ \left(\gamma_d^{(i)} I + \alpha_d^{(i)} 1_d 1_d^\top + d^{-1} \beta_d^{(i)} \bfXi (\bfXi)^\top + \rho_d^{(i)} 1_d^\top  + 1_d (\rho_d^{(i)})^\top + \zeta_d^{(i)} \right)^{-1}\\
    & \quad \times \left( d^{-1} \beta_d^{(i)} \bfXi x_j^{(i)} \right)\left( d^{-1} \beta_d^{(i)} \bfXi x_j^{(i)}\right)^\top \\
    & \quad \times \left(\gamma_d^{(i)} I + \alpha_d^{(i)} 1_d 1_d^\top + d^{-1} \beta_d^{(i)} \bfXi (\bfXi)^\top + \rho_d^{(i)} 1_d^\top  + 1_d (\rho_d^{(i)})^\top + \zeta_d^{(i)} \right)^{-1} \Bigr] \\
    & \leq \frac{\beta_d^{(i)}}{d} \sum_{j} \frac{\lambda_j(\beta_d^{(i)} d^{-1} \bfXi (\bfXi)^\top + \alpha_d^{(i)} 1_d 1_d^\top)}{[\gamma_d^{(i)} + \lambda_j (\alpha_d^{(i)} 1_d 1_d^\top + d^{-1} \beta_d^{(i)} \bfXi (\bfXi)^\top + \rho_d^{(i)} 1_d^\top  + 1_d (\rho_d^{(i)})^\top + \zeta_d^{(i)} )]^2}.
\end{align}
Then, we obtain the statement.
\end{proof}

\section{Additional Remarks of Existing Theory of Kernel Bandits}
\label{sec:summary_kb}
In this section, we summarize the relation between the existing regret analysis of kernel-based contextual bandit algorithms and ours. 
While we believe that our problem setting described in \secref{sec_prelim} is one of the standard settings as used in \citep{neu2024adversarial}, some existing works
consider the different setting about the underlying functions and contexts from ours.
For example, \citet{krause2011contextual} assumes that the underlying function is defined on the product of the arm and context domain, and the common context is generated (or fixed beforehand) for all arms at each time step. 
\citet{valko2013finite} considers a similar setting to ours. However, they assume that the underlying function is common for each arm instead of allowing an arbitrary sequence of context. To describe these different settings consistently, we focus on the setting where the underlying function $f_{\ast}^{(i)}$ is common for each arm; namely, we assume that there exists some function $f_{\ast}$
such that $\forall \in [K], f_{\ast}^{(i)} = f_{\ast}$ in this section hereafter. 
Furthermore, we also assume that the observation noise $\xi(t)$ is $\sigma$-sub-Gaussian random variable as with the setting of \citep{krause2011contextual,valko2013finite}, which is more restricted than our finite variance assumption.

\subsection{Information gain and regret}
Existing theoretical analysis of kernel bandit algorithms relies on the following kernel-dependent 
complexity parameter $\gamma_T$:
\begin{equation}
    \gamma_T = \frac{1}{2} \sup_{\bX \in \mX^{T}} \ln \det (\bI_T + \tau^{-1} K(\bX, \bX)).
\end{equation}
Here, $\tau > 0$ is any fixed positive constant. The quantity $\frac{1}{2} \ln \det (\bI_T + \tau^{-1} K(\bX, \bX))$ is the mutual information gain from the training input $\bX$ under Gaussian process modeling of $f$ with variance parameter $\tau$. For this reason, 
$\gamma_T$ is called maximum information gain (MIG)~\citep{srinivas2012information}.
The MIG is often used to quantify the regret of kernel bandits. For example, CGP-UCB~\citep{krause2011contextual} and Sup-kernel-UCB~\citep{valko2013finite} are known to suffer from $O(\gamma_T\sqrt{T})$ and $O\rbr{\sqrt{\gamma_T T}}$ cumulative regret, respectively. 
In several commonly used kernels, existing works show the upper bound of MIG.
For example, if the underlying kernel $K$ is the RBF class with 
fixed lengthscale parameter $\ell > 0$ (namely, $K(x, x') = \exp(-\|x - x'\|_2^2/\ell)$), the MIG have been shown to increase with $O(\ln^{d+1} T)$~\citep{srinivas2012information,vakili2021information} under $d = \Theta(1)$.
On the other hand, our work focuses on a regime in which $\ell = \Theta(d)$ (Definition~\ref{def:kernel_class}) and $d, T \rightarrow \infty$. 
Regarding the joint dependence of $\ell$, $d$, and $T$, \citet{berkenkamp2019no} shows 
$\gamma_T = O\rbr{\sbr{\ln T^{d+1}(\ln T)}^{d+1}/\ell^{d}}$ (see, Eq.~(24) in \citet{berkenkamp2019no})\footnote{Proposition~2 in \citet{berkenkamp2019no} claims $\gamma_T = O\rbr{(\ln^{d+1} T)/\ell^d}$ for $K(x, x') = \exp(-\|x - x'\|_2^2/\ell)$. However, note that this result is simplified by supposing $d = \Theta(1)$.}. From their result, we obtain $\gamma_T = O\rbr{\ln^{d+1} T}$, which increases super-linearly under any $d = \omega(1)$.
Therefore, in our assumptions, existing $O(\gamma_T\sqrt{T})$ and $O\rbr{\sqrt{\gamma_T T}}$ cumulative regret guarantees become vacuous even if we increase the lengthscale parameter $\ell$ with $d$.

\bibliographystyle{plainnat}
\bibliography{siwazaki_main}

\end{document}